%% file: main.tex
\documentclass[11pt]{article}

\usepackage{amsmath, amsthm}
\input{sections/macros}

\title{Omnipredictors for regression and the approximate rank of convex functions}

\author{Parikshit Gopalan\\
Apple\thanks{{\tt parikg@apple.com}}\and 
Princewill Okoroafor\\
Cornell University\thanks{{\tt pco9@cornell.edu}}\and 
Prasad Raghavendra\\
UC Berkeley\thanks{{\tt raghavendra@berkeley.edu}}
\and   Abhishek Shetty \\
UC Berkeley\thanks{\tt shetty@berkeley.edu}
\and Mihir Singhal\\
UC Berkeley\thanks{{\tt mihirs@berkeley.edu}}
}

\date{}

\begin{document}
\maketitle
\input{abstract}
\thispagestyle{empty}
\newpage
\tableofcontents
\thispagestyle{empty}
\newpage
\setcounter{page}{1}

\input{sections/intro-backup}

\input{sections/approx-reboot}

\input{sections/calibration}

\input{sections/omni.tex}

\input{sections/low_degree}

\input{sections/algorithms}

\input{sections/related}

\bibliography{ref}

\appendix
\input{sections/appendix}

\end{document}

%% file: sections/macros.tex
\usepackage{amsmath, amsthm, amssymb,amsfonts, mathtools, bbm, graphicx}
\usepackage{float}
\usepackage{graphicx}
\usepackage{bm}
\usepackage{xcolor}
\usepackage{algorithm}
\usepackage{algorithmic}
\usepackage{dsfont}
\usepackage{comment}
\usepackage{tikz}
\usepackage[utf8]{inputenc}
\usepackage[english]{babel}
\usepackage{sidecap}
\newcommand{\eat}[1]{}

\usepackage[algo2e,ruled,linesnumbered]{algorithm2e} 

\usepackage{xspace}
\usepackage{url}
\usepackage{caption}
\usepackage{subcaption}
\usepackage{tikz}
\usepackage{paralist}
\usepackage{libertine}
\usepackage[
pagebackref,
colorlinks=true,
urlcolor=blue,
linkcolor=blue,
citecolor=green,
]{hyperref}
\usepackage[nameinlink, noabbrev, capitalize]{cleveref}
\usepackage{enumitem}
\usepackage{todonotes}

\usepackage{color-edits} 
\addauthor{pco}{purple}
\addauthor{pr}{blue}
\addauthor{pg}{red}
\addauthor{ash}{cyan}
\addauthor{ms}{orange}

\usepackage{thm-restate}

\usepackage[margin=1.05in]{geometry}
\usepackage{txfonts}
\usepackage[T1]{fontenc}

\usepackage{autonum}

\usepackage{bm}

\usepackage{cleveref}

\newtheorem{theorem}{Theorem}[section]
\newtheorem{lemma}[theorem]{Lemma}
\newtheorem{definition}[theorem]{Definition}

\newtheorem{proposition}[theorem]{Proposition}
\newtheorem{corollary}[theorem]{Corollary}

\newtheorem{fact}[theorem]{Fact}
\newtheorem{remark}[theorem]{Remark}

\newcommand{\Acal}{\mathcal{A}}
\newcommand{\Bcal}{\mathcal{B}}
\newcommand{\Ccal}{\mathcal{C}}
\newcommand{\Dcal}{\mathcal{D}}

\newcommand{\Fcal}{\mathcal{F}}

\newcommand{\Jcal}{\mathcal{J}}

\newcommand{\Lcal}{\mathcal{L}}
\newcommand{\Mcal}{\mathcal{M}}

\newcommand{\Rcal}{\mathcal{R}}
\newcommand{\Scal}{{\mathcal{S}}}
\newcommand{\Tcal}{{\mathcal{T}}}

\newcommand{\Vcal}{\mathcal{V}}

\newcommand{\Xcal}{\mathcal{X}}
\newcommand{\Ycal}{\mathcal{Y}}
\newcommand{\Zcal}{\mathcal{Z}}
\newcommand{\intzo}{\ensuremath{[0,1]}}
\newcommand{\intpmo}{\ensuremath{[-1,1]}}

\newcommand{\relu}{\mathrm{ReLU}}

\newcommand{\II}{\mathbb{I}} 
\newcommand{\E}{\mathbb{E}}

\newcommand{\reals}{{\mathbb{R}}}

\newcommand{\eps}{\varepsilon}

\newcommand{\abs}[1]{\left| #1 \right|}

\DeclareMathAlphabet\mathbfcal{OMS}{cmsy}{b}{n}

\newcommand{\X}{\mathcal{X}}
\newcommand{\Y}{\mathcal{Y}}
\newcommand{\mC}{\mathcal{C}}
\newcommand{\mD}{\mathcal{D}}

\newcommand{\mS}{\mathcal{S}}
\newcommand{\mT}{\mathcal{T}}

\newcommand{\mB}{\mathcal{B}}
\newcommand{\mF}{\mathcal{F}}

\newcommand{\mL}{\mathcal{L}}

\renewcommand{\Im}{\mathrm{Im}}

\newcommand{\ty}{\mathbf{\tilde{y}}}

\newcommand{\R}{{\mathbb{R}}}
\newcommand{\Z}{{\mathbb{Z}}}
\newcommand{\N}{{\mathbb{N}}}

\newcommand{\mb}[1]{\mathbf{#1}}

\newcommand{\y}{\mathbf{y}}

\newcommand{\x}{\mathbf{x}}

\newcommand{\z}{\mathbf{z}}

\newcommand{\rgta}{\rightarrow}

\newcommand{\lt}{\left}
\newcommand{\rt}{\right}

\newcommand{\zo}{\ensuremath{\{0,1\}}}

\newcommand{\norm}[1]{\left\lVert#1\right\rVert}

\newcommand{\CE}{\mathrm{CE}}

\newcommand{\hkl}{k_{\hat{\ell}}}
\newcommand{\hell}{\hat{\ell}}

\newcommand{\tmD}{\tilde{\mathcal{D}}}

\newcommand{\pmo}{\ensuremath{ \{\pm 1\} }}
\newcommand{\fr}[1]{\ensuremath{\frac{1}{#1}}}

\newcommand{\defeq}{\mathrel{:=}}

\DeclareMathOperator{\rank}{rank}

\DeclareMathOperator*{\argmin}{arg\,min}

\newcommand{\Ber}{\mathrm{Ber}}
\newcommand{\Lip}{\mathcal{F}_\mathrm{lip}}
\newcommand{\Cvx}{\mathcal{F}_\mathrm{cvx}}
\newcommand{\llip}{\mathcal{L}_{\mathrm{lip}}}
\newcommand{\lcvx}{\mathcal{L}_{\mathrm{cvx}}}

\newcommand{\Cvxvec}{\mathbb{L}_{\mathrm{cvx}}}

\newcommand{\tests}{tests}

\newcommand{\PRcomment}[1]{{\color{cyan}[Prasad: #1]}}

\newcommand{\sign}{\mathrm{sign}}

\newcommand{\wl}{\mathsf{WL}}
\newcommand{\ece}{\mathsf{ECE}_{\Scal}}
\newcommand{\estece}{\mathsf{estECE}_{\Scal}}
\newcommand{\recal}{\mathsf{reCAL}_{\Scal}}

\newcommand{\ma}{\mathsf{MA}}


%% file: abstract.tex
\begin{abstract}
Consider the supervised learning setting where the goal is to learn to predict labels $\mathbf y$ given points $\mathbf x$ from a distribution. An \textit{omnipredictor} for a class $\mathcal L$ of loss functions and a class $\mathcal C$ of hypotheses is a predictor whose predictions incur less expected loss than the best hypothesis in $\mathcal C$ for every loss in $\mathcal L$. Since the work of \cite{gopalan2021omnipredictors} 
that introduced the notion, there has been a large body of work in the setting of binary labels where $\mathbf y \in \{0, 1\}$, but much less is known about the regression setting where $\mathbf y \in [0,1]$ can be continuous. The naive generalization of the previous approaches to regression is to predict the probability distribution of $y$, discretized to $\varepsilon$-width intervals. 
The running time would be exponential in the size of the output of the omnipredictor, which is $1/\varepsilon$.

Our main conceptual contribution is the notion of \textit{sufficient statistics} for loss minimization over a family of loss functions: these are a set of statistics about a distribution such that knowing them allows one to take actions that minimize the expected loss for any loss in the family. The notion of sufficient statistics
relates directly to the approximate rank of the family of loss functions. Thus, improved bounds on the latter yield improved runtimes for learning omnipredictors.

Our key technical contribution is a bound of $O(1/\varepsilon^{2/3})$ on the $\epsilon$-approximate rank of convex, Lipschitz functions on the interval $[0,1]$, which we show is tight up to a factor of $\mathrm{polylog} (1/\epsilon)$.  This yields improved runtimes for learning omnipredictors for the class of all convex, Lipschitz loss functions under weak learnability assumptions about the class $\mathcal C$.
We also give efficient omnipredictors when the loss families have low-degree polynomial approximations, or arise from generalized linear models (GLMs). This translation from sufficient statistics to faster omnipredictors is made possible by lifting the technique of loss outcome indistinguishability introduced by \cite{gopalan2022loss} for Boolean labels to the regression setting. 
\end{abstract}

%% file: sections/intro-backup.tex
\section{Introduction}

Loss minimization is the dominant paradigm for training machine learning models. In the supervised learning setting, given a distribution $\mD^*$ on point-label pairs (which we refer to as nature's distribution), we pick a family of hypotheses $\mC$, a loss function $\ell$, and find the hypothesis from $\mC$ that minimizes the expected loss over the distribution. This reduces the task of learning  to an optimization problem over a  parameter space. While this recipe has proven extremely successful, one can ask whether it adequately models a process as complex as learning. 

A weakness of this paradigm is that learning is not robust to the choice of loss function. Different losses result in different optimization problems (which must be solved afresh), and hence typically different optimal hypotheses. One would imagine that each time we minimize a different loss, we learn something new about nature. Is there a universal and rigorous way to synthesize all that we learn into a single model that describes our complete understanding of nature, and does well on all of these losses? Standard loss minimization does not provide a solution for this goal.

Quite often, the exact loss function is not known {\em a priori}.
To illustrate this, we present a simple scenario here.
Suppose that a retailer is building a model to forecast demand for an item in each of its stores. 
\begin{itemize}[leftmargin=*]
\item The retailer has a feature vector $\x$ associated with each store, such as geographical location, foot traffic,  which they use to forecast the demand $p(\x)$ for the item in the store. Based on the forecast $p(\x)$, they decide how much of the item to stock up, which is a number $t \in [0,1]$.  

\item The realized demand is given by $\y^* \in [0,1]$ which represents how much demand for the item there actually was. We assume a joint distribution on $(\x, \y^*)$, but for each $\x$ we only see a single draw $\y^*$ from the joint distribution. 

\item Assume that the retailer sells the item at a fixed retail price of $\$1$ per unit. If the retailer procures the item at a wholesale price per unit $c$ that is determined by the market, and can fluctuate day to day. The loss incurred by the retailer is given by
$\ell^{c}(\y^*,t) = c \cdot t - 1 \cdot \min(t,\y^*)$
\end{itemize}

\eat{
\PRcomment{I am confused here.  Even if we fix the price $p$, we have different loss functions in terms of $\y$, $\ell_{\theta} = \ell(\y,\theta)$ depending on the action $\theta$.  These are all relu functions in $\y$. 
 Naively, even to choose the optimal action, one needs to have all the relu functions here, even if price $p$ is fixed and known.  But why does that example not motivate omniprediction?  Is it because, if we decided to directly learn the  "optimal action function" $\theta$, then it is just loss minimization for that function? 
 }
}

The key observation is that the exact loss function $\ell^c$ depends on the wholesale price per unit $c$ which may be unknown {\em a priori} and probably fluctuates over time. At the time of training the model, the forecaster knows the general shape of the loss function family, but not the exact loss they need to minimize.

The stylized scenario described here is just one example of a recurring theme in applications of forecasting, where the true loss functions are not known {\em a priori}.
This can occur because the loss functions depend on parameters that are not fixed yet.  Alternatively, the same forecasts may be used in many different settings, each of which requires its own distinct loss function.  This raises the question, can we have forecasts which are guaranteed to do well, as measured by any loss drawn from a broad family?

\paragraph{Omniprediction.} This motivated the study of omniprediction, initiated in the work of Gopalan, Kalai, Sharan, Reingold and Wieder \cite{gopalan2021omnipredictors}. 
We will now formally describe the notion of omnipredictors.

\begin{itemize}[leftmargin=*]
\item \textit{Point and label distribution:}
As in supervised learning, the central object being learnt is specified by the {\it nature's distribution}, which is  a joint distribution $\mD^*$ over points $\x \in \X$ and corresponding labels $y^* \in \mathcal{Y}$.

In the demand forecasting example, $\x$ will be the vector of features for each store. The label $\y^*$,  a real number in $[0,1]$, is the demand for the item in the store $\x$, hence $\Ycal = [0,1]$.

\item \textit{Actions and Loss families:} 
The agent, who was the retailer in our running example, intends to use the output of a prediction algorithm towards selecting from a set of actions $\mathcal{A}$.

The loss incurred is a function of the true label and the action chosen, i.e., the loss function family is specified as $\Lcal = \{\ell: \mathcal{Y} \times \Acal \to \R\}$  a family of loss functions.
In this work, the labels $y^*$ would be real valued and therefore $\mathcal{Y} \defeq [0,1]$.

In our running example, the action space consisted of how much of the item to stock up (denoted by $t \in [0,1]$), and the retailer wishes to choose an optimal value for it.
The family of loss functions $\Lcal$ were given by $\ell^c(y,t) = c t - \min(t,y)$ for $c \in [0,1]$.

\item \textit{Predictions and optimal actions:}  
An omnipredictor consists of two efficiently computable functions.  
\begin{itemize}
\item \textit{Prediction function: $p: \X \to \mathcal{P}$} 
Given an input label $\x$, an omnipredictor outputs a prediction $p(\x)$ in some range $\mathcal{P}$, which we will specify shortly. 

The output $p(\x)$ of the omnipredictor  should be thought of as  its prediction of the conditional distribution of the label $\y^* | \x$.

\item \textit{Post-processing function: $k: \mathcal{P} \times \mL \to \mathcal{A}$}

Given the prediction $p(\x)$ and a loss function $\ell \in \mathcal{L}$ in the family, outputs a predicted action $k(p(\x),\ell)$ for the agent.

\end{itemize}

For example, an omnipredictor could simply output a distribution $\mathcal{D}$ over $[0,1]$, which is its prediction for the conditional distribution for $\y|\x$.
In our running example, $\mathcal{D}$ would be its prediction of the probability distribution of demand in the store $\x$.
In this case, the range $\mathcal{P}$ is the space of all probability distributions over $[0,1]$.
Then, the optimal action on a given loss function $\ell \in \mL$ is given by,
\[ k(\mD, \ell) = \arg \min_{\theta \in [0,1]} \E_{ y \in \mD }[\ell(y,\theta)] \]
Predicting the entire distribution is not succinct for real-valued labels. \footnote{We do not make parametric assumptions about the distribution of $\y^*|\x$.} One of the major thrusts of our work will be to find more succinct descriptions of the distribution. 

Crucially, the omnipredictor is trained once and for all without knowing a specific loss. Although the post-processing depends on the loss, it does not require further learning or access to the data set. For instance, take the Boolean setting when $\y^* \in \zo$, and the predictor $p(\x) \in [0,1]$ is an estimate for $\E[\y^*|\x]$. For the $\ell_1$ loss $\ell_1(y, t) = |y -t|$, we take the action $1$ if $p(x) >1/2$ and $0$ otherwise, while for the squared loss, it is the identity function.  

\item \textit{Performance guarantee:}
For both computational and information-theoretic reasons, it is often infeasible to even estimate how far the recommended actions of an omnipredictor are from optimal.
This motivates defining a guarantee for the performance of an omnipredictor relative to the best hypothesis from a concept class.

Fix a concept class of hypotheses,  $\mC = \{ c: \X  \to \Acal\}$ that given the features outputs a recommended action. 
An $(\mL, \mC)$-omnipredictor is one whose expected losses under the distribution $\mathcal{D}$ compete with the best hypothesis in $\mC$ for any loss $\ell \in \mL$.
The power of this guarantee comes from the fact that prediction algorithm $p$ makes predictions without knowing the loss function $\ell$. Yet, these predictions (with the right post-processing function $k$) can compete against the benchmark
\[ \min_{c\in \mC}\E_{(\x, \y^*) \sim \mD}[\ell(y^*, c(\x))] \]
which is very much dependent on the choice of $\ell$.

\end{itemize}

\paragraph{Omniprediction in the Boolean setting}
The work of \cite{gopalan2021omnipredictors} which introduced the notion of omniprediction, studied the Boolean setting where $\y^* \in \zo$. Their starting point is the observation that if we could learn the conditional distribution $\y^*|\x$, then subsequently we could take actions that optimize any loss function, without any further learning or access to the data. 
Since the labels are Boolean, the conditional distribution $\y^*|\x$ is fully described by a single number, namely $p^*(\x) = \E[\y^*|\x]$, and this is what our predictor attempts to predict.

Learning $p^*$ is not feasible in general, for computational and information-theoretic reasons. 
Yet, \cite{gopalan2021omnipredictors} showed that one can efficiently learn $(\lcvx, \mC)$ omnipredictors where $\lcvx$ is the family of convex, Lipschitz loss functions for all $\mC$ that satisfy a basic learnability condition called weak agnostic learnability.  They show this via a surprising connection to a multigroup fairness notion known as multicalibration, introduced by Hebert-Johnson, Kim, Reingold and Rothblum \cite{hebert2018multicalibration}. There has since been a large body of work on this topic, giving omnipredictors for other classes of loss functions  \cite{gopalan2022loss, DworkLLT23}, for constrained optimization \cite{HuNRY23, GGJKMR23}, for other prediction scenarios \cite{GJRR24, KimP23} and proposing stronger notions \cite{GKR23}. Most of this work  considers either the Boolean setting or the multiclass setting (see the discussion of related work).

\subsection{Omnipredictors for regression}

In this work, we present a comprehensive theory of omniprediction for the challenging regression setting where the labels $\y^* \in [0,1]$ are allowed to be continuous. We briefly describe our main contributions.
\footnote{There is ambiguity in the use of the term regression in the literature. In our paper, we will use the term regression to mean that the variable $\y$ is continuous, the loss can arbitrary. In the literature, regression can sometimes refer to (certain) loss minimization problems where $\y \in \zo$, as in logistic regression. }

The natural motivation for considering continuous labels comes from the fact that many real-life forecasting tasks involve predicting real-valued attributes: the amount of rain tomorrow, the price of a stock next week, the temperature of a city in ten years. However, extending prior results on omniprediction to the regression setting has proved challenging (the only prior result we are aware of comes from \cite[Section 8]{gopalan2021omnipredictors}). Indeed, several techniques used in prior work, do not generalize to the more complex continuous setting, see the related work section for a detailed discussion.

\paragraph{Sufficient statistics:}
To develop a theory of omnipredictors for regression, the first question to answer is:  {\em what information should the prediction $p(\x)$ convey about the distribution of the label?} This question has been studied recently in the literature, most relevant to us are the works of \cite{JLPRV21} and \cite{bernoulli}. The answer naturally depends on what the predictions are meant to achieve. For omniprediction, the predictions should enable expected loss minimization for any loss function drawn from the family $\mL$ of loss functions. 

The naive solution would require that $p(\x)$ reveals the conditional probability distribution $\y^*|\x$.
But since $\y^*$ is a continuous random variable, the distribution $\y^* | \x$ may not have a finite description.

 For an arbitrary distribution $\mD$ over $\intzo$, we use the term {\it "statistic"} to refer to the expectation $\E_{\y \sim \mD}[s(\y)]$ of a bounded function $s: \intzo \to [-1,1]$.
For reasons that will be clear later, we limit ourselves to this class of statistics called {\em linearizing statistics} by \cite{bernoulli}. A natural alternative (considered in \cite{JLPRV21, bernoulli}) would be to have the predictions be the statistics associated with the distribution $\y^* | \x$.
This raises the question of when a family of statistics is {\it sufficient} for omniprediction for a loss family $\mL$.
We abstract the requirement in the following definition (the notation ignores the conditioning on $\x$): 

{\em A family of statistics $\Scal$ are said to be {\it sufficient} for the loss family $\mL$ if for any loss $\ell \in \mL$ and distribution $\mD$ over $\intzo$, the value of the statistics $\Scal$ for distribution $\mathcal{D}$ determine the (near) optimal action that (approximately) minimizes the loss. In other words 
\[ k(\mathcal{D}, \ell) = \argmin_{t \in \mathcal{A}} \E_{\y \sim \mD}[\ell(\y, t)], \] 
is determined by the statistics $\Scal$ of distribution $\mD$.
}

If there is a (small) set of {\em sufficient statistics} $\Scal$ for the family $\mL$, then our omnipredictor would try and predict these statistics for every $\x$.   For Lipschitz loss functions, there is a simple set of $1/\eps$   sufficient statistics: the probabilities of the events $\y \in [i\eps, (i+1)\eps)$ for every $i$. We call these the CDF statistics, since they tell us the CDF of $\y$ to within accuracy $\eps$. This gives sufficient information to  minimize the expected loss over actions for any Lipschitz loss function within an additive $\eps$. This was the approach taken in \cite[Section 8]{gopalan2021omnipredictors}, which treats the problem of predicting these probabilities as a multiclass labeling problem with $1/\eps$ labels.

For specific families of loss functions, one could hope to get more succinct statistics. For instance, consider the family of $\ell_p$ loss functions for even $p$, given by $\{ \ell_p(y,t) = (y-t)^p \}$.  
Since this space is spanned by the monomials $\Scal = \{s_i(y) = y^i \;|\; i = 0,\ldots,p\}$, the first $p$ moments are sufficient statistics for this family.

This naturally raises the question: what is the smallest set of sufficient statistics for a family $\Lcal$ of losses? Let us see why it holds the key to more efficient algorithms for omniprediciton.

\paragraph{Omniprediction from indistinguishability:}

The work of \cite{gopalan2022loss}  gives a template for establishing omniprediction by establishing a stronger condition they call loss outcome indistinguishability, which is inspired by the notion of outcome indistinguishability introduced by \cite{OI}. 

Specifically, they show that in the Boolean case, $(\Lcal,\mC)$-omniprediction against a class of loss functions $\Lcal$ and concept class $\mC$ is implied the following properties of the prediction function $p$.
\begin{enumerate}
    \item Calibration: Conditioned on a prediction $p(\x) \in [0,1]$, the expectation is close to the predicted values, i.e., $\E[\y^* | p(\x)] \approx_{\epsilon} p(\x)$.  
    \item Multiaccuracy: the error in prediction $p(\x) - \y^*$ is uncorrelated with a class of \tests\  derived from $\Lcal$ and $\mC$.
\end{enumerate}

The proof is via an indistinguishability argument. They show that one can replace nature's labels $\y^*$ with labels $\ty$ from a simulation that corresponds to the predictor's predictions, without much change in the loss suffered by either the omnipredictor or a hypothesis from $\mC$. The predictor $p$ is Bayes optimal for the simulation, and hence it is an omnipredictor. Indistinguishability  lets us conclude that it is also an omnipredictor for nature's distribution.

Lifting the above result to the regression poses several challenges, which are detailed in Section \ref{sec:omni_Oi}, primarily that there is ambiguity in defining what the simulation being predicted by our predictor is, unlike in the Boolean case. Yet we are able to prove a qualitatively similar statement.  If a predictor $p$ for a family of {\it sufficient} statistics $\Scal$ is calibrated, and is multiaccurate with respect to an associated class of tests derived from $\Lcal,\Scal$ and $\mC$, then it is an $(\Lcal,\mC)$-omnipredictor.
We defer the formal statement of this theorem to \cref{thm:main}. 

Motivated by this sufficient condition for $(\Lcal,\mC)$-omniprediction, we generalize the calibrated multiaccuracy algorithm of \cite{gopalan2022loss} to work in the setting of regression in \cref{thm:main_alg}. The running time is exponential in the size $d$ of the family of sufficient statistics $\Scal$, arising from the need to ensure that our predictions, which take values in $[-1,1]^d$ are calibrated. Consequently, shrinking the size of the family of sufficient statistics results in drastic reductions in running time.

\paragraph{Approximate rank \& sufficient statistics:}

What is the smallest family of sufficient statistics for a given family of loss functions $\Lcal$? 
The answer is directly related to the so-called ``$\eps$-approximate dimension" of a family of functions derived from the loss family $\Lcal$.
\begin{definition}
    Given a family of functions $\Fcal = \{\ell: \intzo \to \intpmo\} $, their $\eps$-approximate dimension denoted by $\dim_\eps(\Fcal)$, is the smallest dimension of a subspace of functions $\mathcal{V}$ such that
    for every $f \in \Fcal$, there exists $\hat{f} \in V$ which is an $\epsilon$-approximation to $f$ in the $\ell_{\infty}$ norm, i.e., $\| f - \hat{f} \| \leq \epsilon$.
\end{definition}
\eat{Recall \pcomargincomment{where am I recalling this from?} that the terminology ``$\epsilon$-uniform approximation'', means an approximation within $\epsilon$ in $\ell_{\infty}$ norm.}

The notion of $\eps$-approximate rank has been studied in the literature, with motivations ranging from communication complexity to approximate Nash equilibria \cite{alon_2009, LeeS09, LSbook, AlonLSV13}.

Given a family of loss functions $\Lcal = \{ \ell: [0,1] \times \Acal \to \R\}$, consider the function family ${\Lcal_t}$ obtained by fixing the actions, i.e., 
\[ {\Lcal_t} = \{ {\ell}_{t} \defeq \ell(\cdot, t) \;|\; \ell \in \Lcal \}\]

Suppose there is a basis of functions $\Scal$ that uniformly approximates the function family ${\Lcal_t}$.
Then for any distribution $\mD$ over $[0,1]$, loss function $\ell \in \Lcal$, and action $t$, the expected loss $\E_{\y \sim \mD}[\ell(t,\y)]$ can be approximately estimated from the expectations of statistics $\{ \E_{\y \sim \mD}[s(\y)] \}$. This allows us to choose the best action for each $\ell$ (at least information-theoretically), as required by the definition of sufficient statistics.  
Therefore there is a tight connection between {\it sufficient statistics for loss family $\Lcal$} and the $\epsilon$-approximate dimension of the corresponding function family ${\Lcal_t}$: an $\eps$-approximate basis for the latter gives us functions whose expectations are sufficient statistics for the former.

\eat{
\begin{remark}
Although this connection is fairly tight, there are two caveats.  First, the functions used to $\epsilon$-uniformly approximate $Lcal_t$ are not necessarily bounded, while statistics are to be bounded.   Second, in our application, the coefficients needed to express a loss function $\ell$ in terms of $\Scal$ need to be bounded, while approximate dimension requires no such bound apriori.  See definition 
\end{remark}
\ashmargincomment{I am not sure this remark is that necessary as long as $\Lcal$ itself is bounded in light of \cref{cor:bound_coeff}}
}

Thus upper bounds on approximate dimension of loss families lead to upper bounds on the complexity of learning omnipredictors.  Our next contribution is to show that many natural loss families admit non-trivial uniform approximations.

\eat{
As a warmup, consider the family of $\ell_p$ loss functions for even $p$, given by $\{ \ell_p(y,t) = (y-t)^p \}$.  
Since this space is spanned by the monomials $\Scal = \{s_i(y) = y^i | i = 0,\ldots,p\}$, we get a polynomial time for omniprediction for every constant $p$.

\PRcomment{theorem statement?}

For instance, to minimize the squared loss, knowing $\E[\y]$ \pcomargincomment{Is this right? shouldn't it be $y,y^2$?}is clearly sufficient. For $\ell_p(y,t) = (y - t)^p$ for even $p$, knowing the first $(p -1)$ moments suffice. 
}

\paragraph{Approximate rank of convex Lipschitz functions:}

A recurring property of loss functions that arise in a myriad of contexts is {\it convexity}.  
For example, the loss family in the example of demand forecasting for a retailer were convex functions over $[0,1]$. 
This makes it especially important to understand the approximate dimension of the space of convex Lipschitz functions over $[0,1]$ \footnote{Lipschitzness is a natural constraint here to make the question of $\epsilon$-approximations invariant to scaling.}.

In the absence of convexity, if one considers the family of all Lipschitz functions denoted $\Lip$, it is easy to show that $\dim_{\eps}(\Lip) = \Theta(1/\epsilon)$.
The upper bound follows by a straightforward basis consisting of indicators of intervals $1[y \geq i \epsilon]$ for $i \in \{0,1,\ldots,1/\eps\}$.
Using a linear algebraic argument, one can show that $1/\eps$ statistics are indeed necessary. 

It is natural to ask if convexity leads to better approximations to the functions.
Our main technical result shows that the answer is yes, and in fact,  we exhibit construction of a set of $\tilde O(1/\eps^{2/3})$ statistics that suffice for every function in $\Cvx$, the family of bounded, Lipschitz, convex functions on $[0,1]$. 
Moreover this bound is essentially tight: we show a lower bound of $\Omega(1/\eps^{2/3})$ which holds even for the family $\{\relu_{i\eps}(y)\}_{i=1}^{1/\eps}$ which is a subset of $\Cvx$. 

An interesting implication is that the number of statistics required to approximate (the expectations of) $\ell_1$ losses of the form $|y-t|$ is very different from the number required for $\ell_2$ losses $(y-t)^2$. We show that the former require $\Omega(1/\eps^{2/3})$ statistics, whereas for the latter we only need a constant number of statistics: $\E[y]$, and $\E[y^2]$ suffice.

\paragraph{Omnipredictors for loss families}

Using the technical  above, we give the first efficient omnipredictors for several important families of loss functions. 

\begin{itemize}
    \item A main application of our result on the approximate rank of convex Lipschitz functions is an omnipredictor for the family of all Lipschitz, convex loss functions in $\y$ which we refer to as $\lcvx$. 
    We show in \cref{thm:cvx_omni} that a predictor that is calibrated for a family of statistics that arise from our approximation theorems and is multiaccurate with respect to bounded postprocessings of $\Ccal$ is a $( \lcvx , \Ccal , \epsilon) $-omnipredictor and can be computed in time $ \exp( \tilde{O}( \epsilon^{-2/3} ) ) $ time. 
    This is a significant improvement over the  $ \exp( \tilde{O}( \epsilon^{-1} ) )   $ time algorithm one would arrive at by predictor that is calibrated with respect to the CDF statistics. The result of \cite[Section 8]{gopalan2021omnipredictors} give a running time of $\exp(O(1/\eps))$ again using CDF statistics, but with the requirement that the loss is convex in $t$ (it need not be convex in $y$). Their result requires multicalibration for $\mC$, whereas we require calibrated multiaccuracy for postprocessings of $\mC$.

    \item For the class of $\ell_p$ for even $p \leq d $, we show, in \cref{thm:low_degree}, that calibration with respect to the moment statistics of degree $\tilde{O} (\sqrt{d} ) $ and multiaccuracy with respect to polynomials postprocessings of $\Ccal$ leads to an omnipredictor. This leads to an omnipredictor that runs in time $ (1/ \epsilon)^{ \tilde O({ \sqrt{d} }) } $ which improves upon the naive $ (1 / \epsilon)^d$ algorithm that predicts the first $d$ moments. 
    
    \item For the class of losses corresponding to generalized linear models with respect to a family of statistics $\Scal$, we show, in \cref{thm:GLM}, that a predictor $p$ that is calibrated with respect to $\Scal$ and multiaccurate with respect to $\Ccal$ is an omnipredictor. 
    This leads to a $ (1/ \epsilon)^d $ time algorithm for producing an omnipredictor. 
    This result should be contrasted with the earlier results due to the fact we need only access to a weak learner for the original class $\Ccal$ as opposed to postprocessings of it. This generalizes a result of \cite{gopalan2022loss} in the Boolean setting. 
\end{itemize}

\subsection{Overview of technical contributions}

In this section, we highlight the main new technical contributions of this work. We first discuss the approximate rank of convex functions, and then our generalization of loss outcome indistinguishability to real valued labels.

\subsubsection{Approximating univariate convex functions.}
Recall that $\Cvx$ denotes the space of convex $1$-Lipschitz functions on the interval $[0,1]$. We prove our bound on the approximate rank by a sequence of reductions. Let $g \in \Cvx$ be a convex function that we wish to approximate uniformly. 

\paragraph{Reduction to discrete functions.} We place a $\delta$-grid on  $[0,1]$ and use the piecewise linear approximation to $g$ to reduce the problem to a discrete problem of approximating functions $f:[m] \to \R$ where $m = 1/\delta$. The Lipschitzness translates to the fact that the first finite differences of these functions are bounded, and convexity corresponds to positivity of the second finite difference.

\paragraph{Reduction to the $\relu$ functions.} The $\relu$ family of functions mapping $[m]$ to $\Z$ is defined as $\relu_i(x) = 0$ for $x < i$ and $x -i$ for $x \geq i$. We prove a discrete Taylor theorem to show that $\eta$-uniform approximations to these functions imply $O(\eta)$ uniform approximations to all  functions from $\Cvx$. This step is similar in spirit to \cite[Theorem 8]{ucal}, which  implies that the family of functions $|x -i|$ is a basis for all convex functions with small $\ell_1$ norm.\footnote{Their motivation, which is quite different form ours, is from the new notion of $U$-calibration that they define.}

\paragraph{From $\relu$ to intervals.} If we were to form a basis which simply contained all these $\relu_i$, then we would end up with a basis of size $m$, which has the same size as the trivial basis of size $1/\delta$. It turns out, though, that it is possible to approximate the $\relu$ functions using a smaller basis.
We first observe that
\begin{equation}\label{eq:relu-consecutive}
 \relu_a(x) - \relu_{a+1}(x) = 1[x \geq a],
 \end{equation}
where $1[x \geq a]$ is the indicator of $\x \geq a$, or equivalently the indicator function of the interval $[a,m]$. More generally, the differences between $\relu$ functions can be expressed as sums of indicators of intervals. It turns out that it is possible to effectively approximate the interval functions. Therefore, we use the following natural strategy to construct a basis approximating $\relu$ functions.

Our final basis will combine evenly spaced $\relu$ functions with a basis for the interval functions. Specifically, for $t=m^{1/3}$, we add $\relu_{i \cdot t}$ to the basis for each $i$. We take the union of this with a $(1/m^{1/3})$-approximate basis for all interval functions. Then, we can approximate any given $\relu$ function by starting with the $\relu$ function at the nearest multiple of $t$, and then adding approximations to interval functions. For the appropriate basis of interval functions, this will give the desired basis of size $\tilde O(m^{2/3}) = \tilde O(1/\delta^{2/3})$ for all convex functions.

\paragraph{Approximating intervals.} The final step is to approximate all interval functions. By the dyadic decomposition of intervals, it suffices to consider only dyadic intervals. For simplicity, consider all intervals containing a single point. A low rank approximations to all such functions is equivalent to a low rank approximation to the $m \times m$ identity matrix. Approximate low-rank factorizations of the identity matrix arise in the context of Johnson-Lindenstrauss lemma.  They can be explicitly constructed using codewords from a binary code of distance $1/2 - \mu$ and rate $\Omega(1/\mu^2)$. It is known that an $\eps$-approximation can be obtained using rank $\log(m)/\eps^2$. Matching lower bounds on the rank are proved by Alon \cite{alon_2009}. 

\eat{

    \paragraph{Putting it all together} Our final basis combines evenly spaced $\relu$ functions with a basis for the interval functions. Specifically, we do the following:

    \begin{enumerate}
    \item Pick a subset of $\relu$ functions, specifically, $\relu_i$ for offsets $i$ at regular spacing in $[m]$, namely,
    \[\{ \relu_{i\cdot t}\;|\; i = 0,1,\ldots,m/t-1\},\]
    for $t = m^{1/3}$ (we again assume for convenience that $t$ divides $m$). 
    
    \item Include $1/m^{1/3}$ uniform approximations of interval functions $\II_{a,m}(y) = 1\left[ y \in [a,m]\right]$ for all $a \in [m-1]$.
    
    \item For each $k \in [m-1]$, that is not a multiple of $t$, reconstruct $\relu_k$ from the previous multiple of $t$ by adding interval functions.
    \end{enumerate}

    This gives the desired basis of size $\tilde O(m^{2/3}) = \tilde O(1/\delta^{2/3})$ for all convex functions.
}

\subsubsection{Loss outcome indistinguishability for predicting statistics} 
Outcome indistinguishability (OI) was introduce in \cite{OI} for the Boolean setting, and generalized to regression in \cite{bernoulli}. The work of \cite{gopalan2022loss} connected it to omniprediction in the Boolean seeting, introducing the notion of loss OI. 
The generalization of loss OI \cite{gopalan2022loss} to real values is not straightforward. For the sake of concreteness, assume that the statistics we predict are the first $d$ moments $\{y^{i}\}_{i \in [d]}$. In the Boolean setting, when we predict $p(\x) = 0.7$, it is clear that we mean $\y^*|\x$ is drawn from the Bernoulli distribution with parameter $0.7$. When we predict the first $d$ moments of a distribution, there might be many distributions matching those moments. Or there might not be any! The first $d$ moments have to satisfy various moment inequalities, which our predictor would need to satisfy in order to make predictions that a realizable via some distribution. For the moments, there are indeed efficient (SDP-based) methods to ensure feasibility of predictions (See e.g. \cite[Theorem 3.1]{schmdgen2020lectures}).
For other families of statistics $\mS$, such characterizations might not exist or might be computationally infeasible.

We require our predictors to satisfy two conditions: calibration and multiaccuracy, as in \cite{gopalan2022loss}. Calibration requires that conditioned on a prediction, the expectations are close to the predicted values, whereas multiaccuracy requires that the errors in prediction $p_i(x) - s_i(\y^*)$ are uncorrelated with a class of \tests\ derived from $\Lcal$ and $\mC$.  

For the analysis, we define a {\em simulation} distribution $(\x, \ty)$ where $\ty|\x \sim \y^*|p(\x)$. In a sense, this is the random variable whose statistics our predictor $p$ predicts (with some error). This is different from  Boolean setting \cite{OI, gopalan2022loss}, where the simulation is based on the predictor alone, and is independent of the distribution $\mD^*$ that is being learnt. It is more reminiscent of the view of \cite{gopalan2021omnipredictors} for the Boolean case, who view predictors as partitions of the space into level sets, with the canonical prediction which is the expectation over the level set. The simulation shows that calibration approximately solves the feasibility issue above, since if a predictor is $\alpha$ calibrated, then on expectation over $\mD^*$, it holds that 
\[ |p_i(\x) - \E[s_i(\ty)]| = |p_i(\x) - \E[s_i(\y^*)|p(\x)]| \leq \alpha. \]

With this definition in place, we can deduce omniprediction using a similar high-level strategy to the one used in \cite{gopalan2022loss}: for any loss $\ell \in \Lcal$, we show that the expected loss of the omnipredictor, where we make decisions based on $p(\x)$ does not change if the labels are drawn from $\y^*$ or $\ty$, nor does the expected loss suffered by any hypothesis in $\mC$. The implementation departs from the Boolean case. There the first condition (called decision OI) is guaranteed by calibration alone, and the second (called hypothesis OI) by multiaccuracy. In our setting, we do not have similar decomposition. Showing that the expected loss for $c \in \mC$ does not change much when we switch between $\ty$ and $\y^*$ requires both calibration and multiaccuracy, this stems from the fact that our simulation is dependent on both $p$ and the distribution $\mD^*$. 

\paragraph{Algorithm for calibrated multiaccuracy.} We present an algorithm that achieves calibrated multiaccuracy for $\mS$-predictors, assuming access to  a suitable weak agnostic learner. This generalizes the algorithm from \cite{gopalan2022loss}. Note that calibrated multiaccuracy is much weaker than multicalibration, and hence is more efficient to achieve. The running time for calibration is exponential in $d$, the number of statistics, since verifying if a $d$-dimensional predictor is calibrated requires $\exp(d)$ samples. Thus efficient algorithms crucially rely on the cardinality $d$ of the sufficient statistics being small. 

The multiaccuracy is for a family of \tests\ $\Bcal = \{f_\ell \circ c: c\in \mC\}$ where $f_\ell: [-1,1] \to [-1,1]$ is a family of bounded post-processing functions derived from the loss family $\Lcal$.
The family of such tests also shows up in the work of \cite{gopalan2022loss}, who refer to it as $\mathrm{level}(\mC)$. The reason is that it is a family of post-processing functions, so its level sets are (unions of) the level sets of $\mC$. One can equivalently think of $\mathrm{level}(\mC)$ as the closure of $\mC$ under post-processing functions. As observed by \cite{gopalan2022loss}, when $\mC$ is the family of decision trees of bounded size or a family of Boolean functions,  $\Bcal$ and $\mC$ are essentially the same. Similarly, when the action space $\Acal$ is discrete, $\Bcal$ is just a mapping of actions into real space. But for other hypotheses classes like low-degree polynomials, $\Bcal$ might be richer than $\mC$, so the problem of weak agnostic learning for it is harder.

It was already observed in the work of \cite{gopalan2022loss}, Calibrated multiaccuracy is computationally much more efficient than multicalibration in. This difference is even more pronounced for statistic predictors. The exponential dependence on $d$ in the running time for calibrated multiaccuracy arises from the need for calibration. The number of calls to the weak learner is (only) polynomial in $d$. By contrast, achieving  multicalibration for a statistic predictor (as in \cite{JLPRV21, gopalan2021omnipredictors}) requires an exponential number of calls to the weak learner using the best known algorithms for multicalibration. Thus, this presents an improvement over using multicalibration (assuming weak learning is approximately equally difficult over $\mC$ and $\mT$).

\subsection{Organization}
We present our results on the approximate rank of convex functions in \Cref{sec:convex}. The construction is self-contained and does not use any machinery beyond JL matrices and does not require any knowledge of multigroup fairness or omniprediction.
In \Cref{sec:sufficient}, we formally introduce the notion of sufficient statistics for families of loss functions and define calibration and multiaccuracy for statistic predictors. In \Cref{sec:omni_Oi}, we show how to obtain omniprediction from loss outcome indistinguishability of statistic predictors. In \Cref{sec:omni_app}, we obtain omniprediction guarantees for convex lipschitz losses, low-degree polynomials, and generalized linear models. In \Cref{sec:omni_alg}, we present our algorithm for achieving calibrated multiaccuracy. Further discussion of related works can be found in \Cref{sec:rel_works}.

%% file: sections/approx-reboot.tex
\section{Uniform approximations to convex functions}
\label{sec:convex}

We will be interested in uniform approximations of functions from a (possibly infinite) family $\Fcal$ using linear combinations of functions from a (small) finite basis $\Scal = \{s_i\}_{i=1}^d$. We will use the following definition.

\begin{definition}[$\eps$-approximate basis, $\eps$-approximate dimension] \label{def:approx}
    Fix a family of functions $\Fcal = \{f: \mathcal{D} \to \R\} $ of functions over a domain $\mathcal{D}$.  A basis $\Scal = \{s_i: \mathcal{D} \to [-1,1]\}$ is said to $\eps$-approximately span, or be an $\eps$-approximate basis for, the family $\Fcal$ if for every $f \in \Fcal$, there exists $\{r_i \in \R\}_{i =0}^d$  such that  
    \begin{align}
        \left\| r_0 + \sum_{i =1}^{|\Scal|} r_i s_i - f\right\|_{\infty} \leq \eps. \\
    \end{align}
    Moreover, we define the $\eps$-approximate dimension of $\Fcal$, denoted $\dim_\eps(\Fcal)$, to be the smallest size of any $\eps$-approximate basis of $\Fcal$.
\end{definition}

A function family of key interest is the space of convex 1-Lipschitz functions over an interval, say $[0,1]$.  We will use $\Cvx$ to denote this family of functions.

In this work, we obtain tight upper and lower bounds (up to logarithmic factors) on the size of basis that uniformly approximates $\Cvx$.  Specifically, we will show the following result.

\begin{theorem}\label{thm:cvx-approximation}
For every $\delta > 0$, we have
\begin{align}
\Omega\left(\frac{1}{\delta^{2/3}}\right) \le \dim_\delta (\Cvx) \le O\left(\frac{1}{\delta^{2/3}} \cdot \log^3(1/\delta)\right).
\end{align}

\end{theorem}

We will prove this result in the coming subsections.

\subsection{Constructing a basis}
We will explicitly construct a basis $\Scal$ using a series of reductions starting with the function family $\Cvx$ to progressively simpler families.  
The first step in this series of reductions is discretization.

\paragraph{Discretization and scaling}

Since the functions in $\Cvx$ are $1$-Lipschitz, they can be approximated by piecewise constant functions by a straightforward discretization.

For notational convenience, we will assume that $\delta = 1/m$ where $m \in \Z^+$ is a power of $2$.  Consider the $\delta$-grid on the interval $\intzo$ consisting of the points $G_\delta = \{ i \delta\}_{i= 0}^m$. 
Given a function $g: [0,1] \to \R$, one can construct a piece-wise constant function $\hat{g}$, by 
setting $\hat{g}(x) = g\left( \delta \lfloor \frac{x}{\delta} \rfloor \right)$.

\begin{lemma}[Discretization]\label{lem:discretization}
Given a $1$-Lipschitz convex function $g: [0,1] \to \R$, construct a piecewise constant function $\hat{g}$ by fixing $\hat{g}(y) = g\left(\delta \cdot \lfloor \frac{y}{\delta} \rfloor \right)$ for all  $y \in [0,1]$.  Then,
$|\hat{g}(y) - g(y) | \leq  \delta$ for all $y \in [0,1]$.
\end{lemma}
\begin{proof}
For any $i \in \{0,1,\ldots,1/\delta\}$ and any $y \in [i \delta, (i+1) \delta]$,
\[ |g(y) - \hat{g}(y)| \leq |g(i\delta) - \hat{g}(i\delta)| + |g(y) - g(i \delta)| + |\hat{g}(y) - \hat{g}(i \delta)| .\]
Since $g$ is $1$-Lipschitz, we have $|g(y) - g(i \delta)| \leq |y - i \delta| \leq \delta$.  By definition,  $\hat{g}(y) - \hat{g}(i \delta) = 0$ and  $g(i\delta) - \hat{g}(i \delta) = 0$ and thus we have the result.
\end{proof}

Clearly, approximating these piecewise constant functions $\hat{g}$ reduces to approximating their values on the $\delta$-grid $G_{\delta}$.  
For notational simplicity, we will scale the domain by a factor of $\frac{1}{\delta}$, and consider the related approximation problem for vectors, i.e., functions over $[m]=\{0, 1, \dots, m\}$.

The functions $\hat g$ already exist in a vector space of size $m+1 = 1/\delta+1$, so it already follows that there is a trivial $(\delta+1)$-approximate basis for $\Cvx$. However, we will show that we can actually construct an $\tilde O(1/\delta^{2/3})$-basis.

To this end, let us begin by defining the difference operators for functions on $[m]$.
Define the first difference operator $\Delta$ on functions over $[m]$ by setting 
\begin{align}
    \Delta f(y) &= f(y + 1) - f(y) \label{eq:first-diff},
\end{align}
for all $f:[m] \to \R$. (Note that the domain of $\Delta f$ is then $[m-1]$.)
Similarly, define the second difference operator $\Delta^2$ as
\begin{align}
    \Delta^2 f(y)  = (\Delta \circ \Delta) f(y) = f(y + 2) + f(y) - 2 f(y + 1). \label{eq:second-diff}
\end{align}

Let $\Cvxvec$ denote the set of vectors in $\R^m$ that satisfy a discrete version of convexity and $1$-Lipschitzness. 
\begin{align}
\Cvxvec = \left\{ f : [m] \to \R  \ \middle|\  
\begin{aligned}
&(\textbf{Convexity})\ \  \Delta^2 f(i) =  f(i+2) - 2 f(i+1) + f(i) \geq 0 & \forall i \in [m-2]  \\
&(\textbf{$1$-Lipschitz})\ \ |\Delta f(j)| = |f(j+1) - f(j)| \leq 1 & \forall j \in [m-1] 
\end{aligned}
\right\}
\end{align}

For any convex function $f: [0,m] \to \R$, its restriction to integers satisfies the discrete convexity property above.
Also, observe that uniform approximations the set of vectors $\Cvxvec$ yields corresponding approximations to the space of convex functions $\Cvx$.  Formally, we have:
\begin{lemma} \label{lem:cvx-discretization} For all $\eta, \delta > 0$,
suppose that there is a basis $\hat{\Scal}$ of functions over $[m]$ which $\eta$-approximately spans the space $\Cvxvec$.
Define the corresponding basis $\Scal$ of functions over $[0,1]$ as follows:
\[\Scal = \left\{ g:[0,1] \to \R \ \middle|\  g(x) = \delta \cdot \hat{g}\left(\lfloor \tfrac{x}{\delta} \rfloor\right), \hat g \in \hat{\Scal}\right\}.\]
Then, $\Scal$ $(\delta (1+ \eta))$-approximately spans $\Cvx$.
\end{lemma}
The above claim follows immediately from \cref{lem:discretization} and the correspondence between piecewise constant functions and vectors.
The rest of the section will be devoted to constructing a basis $\hat{\Scal}$ which $\Theta(1)$-approximately spans $\Cvxvec$.

\paragraph{From piece-wise linear convex functions to ReLU}

In the next step, we show that all functions in $\Cvxvec$ are well-approximated by the class of ReLU functions.  We begin by recalling the definition of the ReLU function family.  For each $i \in [m]$, define $\relu_i(y)$ as
\begin{align} \relu_i(y) = 
\begin{cases} 0 & \text{for}\ y < i\\
y -i & \text{for} \ y \geq i
\end{cases}.
\end{align}

\eat{
It follows that for $f \in \Lip$, $\norm{\Delta_\delta f}_\infty \leq \delta$, and $\norm{\Delta^2_\delta f}_\infty \leq 2\delta$.  Additionally if  $f \in \Cvx$, then $\Delta_\delta f(y)$ is increasing with $y$ so  
$\Delta^2_{\delta} f(y) \geq 0$.
}

\eat{
For $i \in \{0,\ldots, m-1\}$, we denote the slope of $\hat{f}_\delta$ in the interval $[i\delta, (i+1)\delta]$  by
\[ \Delta f(i) = \frac{f((i+1)\delta) - f(i\delta)}{\delta}. \] 
For $f \in \Lip$, $|\Delta f(i)| \leq 1$ for every $i$. For notational convenience, we set $\Delta f(-1)  = 0$. }

Next we show that $\Cvxvec$ lies in the span of the $\{\relu_i\}_{i \in [m]}$ via the following expansion for functions, which is essentially the discrete version of a Taylor series.

\begin{lemma}[Discrete Taylor series expansion]
\label{lem:relu-basis}
    Every function $f : [m] \to \R$ can be written as
    \begin{align}
    f(y) = f(0) +  (\Delta f(0)) \cdot y + \sum_{i=0}^{m-2} \left(\Delta^2 f(i)\right) \cdot  \relu_{i+1}(y). 
    \end{align}
\end{lemma}
\begin{proof}

For $y = 0$, it is easy to see that both sides equal $f(0)$. Now suppose that $y \ge 1$. Then, expanding the right-hand side of the above equation,
    \begin{align}
    &\phantom{{}={}} f(0) + \left(\Delta f(0) \right) \cdot y +  \sum_{i=0}^{m-2} \left(\Delta^2 f(i)\right) \relu_{i +1}(y)\\
    &= f(0) + \left(\Delta f(0)\right) \cdot y  + \sum_{i=0}^{y-2} \left(\Delta f(i+1) -\Delta f(i)\right)(y - i - 1)\\
    &= f(0) + \sum_{i=0}^{y-1} \Delta f(i)\\
    &= f(0) + \sum_{i=0}^{y-1} (f(i+1) - f(i))\\
    & = f(y).
    \end{align}
\end{proof}

Thus, every function on $[m]$ can be approximated by a linear combination of $\relu$ functions. Moreover, it turns out that for functions in $\Cvxvec$, the sum of the coefficients of the $\relu$ functions is actually $O(1)$. This allows us to conclude that approximately spanning $\Cvxvec$ is actually equivalent to just approximately spanning the ReLU functions:

\begin{corollary} \label{cor:relu}
Suppose that a $\Scal$ is an $\eta$-approximate basis for the family $\{\relu_i \;|\; i \in [m-1]\}$. Then, $\Scal$ also $(3\eta)$-approximately spans all of $\Cvxvec$.
\end{corollary}

\begin{proof}
Notice that $\relu_0(y) = y$, so \cref{lem:relu-basis} lets us write $f$ as a linear combination over $\relu_{i}$ for $i \in [m-1]$.

For each $i$, let $\hat{\relu}_i$ denote the $\eta$-approximation to $\relu_i$ given by the basis $\Scal$, i.e., $\|\relu_i - \hat{\relu}_i\|_{\infty} \leq \eta$ and $\hat{\relu}_i \in \mathrm{Span}\{ \Scal \}$.

For any convex function $f$, define its approximation $\hat{f}$ by
\begin{equation} \label{eq:relu-approx}
\hat{f}(y) = f(0) +  (\Delta f(0)) \cdot y + \sum_{i=0}^{m-2} \left(\Delta^2 f(i)\right) \cdot  \hat{\relu}_{i+1}(y). 
\end{equation}
Thenm
\begin{align}
\| f- \hat{f}\|_{\infty} & = \max_{y} \left | (\Delta f(0)) \cdot (\relu_0(y) - \hat{\relu_0(y)} ) + \sum_{i=0}^{m-2} \left(\Delta^2 f(i)\right) \cdot (\relu_{i+1}(y) - \hat{\relu}_{i+1}(y)) \right |  \\
& \leq  |\Delta f(0)| \cdot \| \relu_0 - \hat{\relu_0}\|_{\infty} + \sum_{i = 0}^{m-2} |\Delta^2 f(i)| \cdot \| \relu_{i+1} - \hat{\relu}_{i+1}\|_{\infty} \\
& \leq \eta \cdot \left( |\Delta f(0)| + \sum_{i = 0}^{m-2} |\Delta^2 f(i)|\right)
\end{align}

Since $f$ is $1$-Lipschitz, $|\Delta f(0)| \leq 1$.  Moreover, since $f$ is convex, $\Delta^2 f$ is positive everywhere, so we can bound
\begin{align}
\sum_{i=0}^{m-2} |\Delta^2 f(i)|  &=  \sum_{i=0}^{m-2} \Delta f(i+1) - \Delta f(i)  \\
&= \Delta f(m-1) - \Delta f(0)\\
& \leq 2.
\end{align}
Therefore, $\| f- \hat{f}\|_{\infty} \le 3 \eta$, so $f$ is approximated with an error of $3 \eta$ by the basis $\Scal$.
\end{proof}

\paragraph{From ReLU to Interval functions}

Via \cref{cor:relu} and \cref{lem:cvx-discretization}, our problem is reduced to finding uniform approximations to the class of $\relu$ functions $\{\relu_i \;|\; i \in [m-1]\}$.

If we were to form a basis which simply contained all these $\relu_i$, then we would end up with a basis of size $m$, which has the same size as the trivial basis of size $1/\delta$. It turns out, though, that it is possible to approximate the $\relu$ functions using a smaller basis.
We first observe that
\begin{equation}\label{eq:relu-consecutive}
 \relu_a(y) - \relu_{a+1}(y) = 1[y \geq a],
 \end{equation}
where $1[y \geq a]$ is the indicator of $y \geq a$, or equivalently the indicator function of the interval $[a,m]$.

More generally, the differences between $\relu$ functions can be expressed as sums of indicators of intervals.  
Formally, let $\II_{a, b}(y) = 1\left[y \in [a, b]\right]$ denote the indicator function for the interval $[a, b]$. Then, we have the following proposition.
\begin{proposition}
 \label{prop:relu-int}
    For $0 \leq j \leq k \leq m$ and $y \in [m]$
    \begin{equation}\label{eq:reludiff}
     \relu_{j}(y) - \relu_{k}(y) =  \sum_{i= j+1}^{k}\II_{i, m}(y).  
    \end{equation}
\end{proposition}
Arguably, the class of interval functions are a simpler class than $\relu$ functions, and conceivably, they admit better uniform approximations than $\relu$ functions.  Therefore, we use the following natural strategy to construct a basis approximating $\relu$ functions.

\begin{itemize}
\item Pick a subset of $\relu$ functions, specifically, $\relu_i$ for offsets $i$ at regular spacing in $[m]$, namely,
\[\{ \relu_{i\cdot t}\;|\; i = 0,1,\ldots,m/t-1\},\]
for some $t$ (we again assume for convenience that $t$ divides $m$). Include all these functions in the basis.

\item Include uniform approximations of interval functions $\II_{a,m}(y) = 1\left[ y \in [a,m]\right]$ for all $a \in [m-1]$.

\item For each $k \in [m-1]$, that is not a multiple of $t$, reconstruct $\relu_k$ from the previous multiple of $t$ by adding interval functions, using \eqref{eq:reludiff}.
\end{itemize}
The following proposition follows immediately from \eqref{eq:reludiff}, and we include the proof here for the sake of completeness.
\begin{proposition} \label{prop:build-basis}
Suppose that $\Scal$ is an $\eta$-approximate basis for the class of interval functions $\II^m = \{\II_{a,b}\;|\; a, b \in [m]\}$. Then, $\Scal \cup \{\relu_{it}\;|\;i \in [m/t-1] \}$ is an $(\eta t)$-approximate basis for the class of all ReLU functions $\{ \relu_{i} \;|\; i \in [m-1] \}$.
\end{proposition}
\begin{proof}
For $k \in [m]$, let $j = t \cdot \lfloor k/t \rfloor$ denote the largest multiple of $t$ less than or equal to $k$.  We can express $\relu_k$ as $\relu_k = \relu_{j} - \sum_{i=j+1}^k \II_{i,m}$.
Suppose $\sum_{s \in \Scal} r^{(i)}_s \cdot s$ is a $\eta$-approximation to $\II_{i,m}$ in $\ell_{\infty}$ norm, for each $i$. Then,
\[ \left\| \sum_{i = j+1}^k \II_{i,m} - \sum_{s \in \Scal} \left(\sum_{i=j+1}^k r^{(i)}_s\right) \cdot s\right \|_{\infty} 
\leq \sum_{i = j+1}^k \left\|\II_{i,m} - \left(\sum_{s \in \Scal}  r^{(i)}_s \right) \cdot s\right \|_{\infty} 
\leq |k-j| \cdot \eta 
\leq t \eta \ . \]
\end{proof}

\eat{\begin{proof}
    Observe that
    \begin{align}
        \relu_{i-1}(y) - \relu_{i}(y)   = \begin{cases} 0 \ &  y \leq (i - 1)\\
        y - i \ & y \in [(i -1)\delta, i\delta]\\
        1 \ & \ y \geq i\delta.
        \end{cases}
    \end{align}
    This implies that
    \begin{align}
        \II_{i, m}(y)  = \relu_{i-1}(y) - \relu_{i}(y) 
    \end{align}
\end{proof}
}

\paragraph{Approximately spanning intervals via JL matrices}

Let $\II^{m}$ denote the set of all interval functions --- that is, let $\II^{m} = \{ \II_{a,b} | a,b \in [m]\}$. In this section, we will construct a small basis that uniformly approximates all interval functions.

The set of dyadic intervals will serve as a stepping stone towards approximating all intervals.
The subset of dyadic intervals consists of all intervals $\II^{m}_{j,k}$ where $ j = i2^h$ and $k = (i+1)2^h -1$ for integers $i, h$, we denote it by $\mathbb{D}^{(m)}$. Note that every interval in $\II^{m}$ can be written as the disjoint union of $2\log(m)$ intervals from $\mathbb{D}^{(m)}$: 

\begin{proposition}
\label{prop:dyadic}
    Every interval function in $\II^{m}$ can be expressed as a sum of at most $2 \log(m)$ dyadic intervals from $\mathbb{D}^m$. 
\end{proposition}

We will use low-rank approximate factorizations of the identity matrix as described in the following lemma.
\begin{lemma}[Low-rank factorization of the identity matrix]
\label{lem:codes}
    There exists $c >0$ such that the following holds for all $\mu > 0$ and $n$.  There exists
    an $n \times k$ matrix $V$ where $k= c\log(n)/\mu^2$ such that  
    \begin{align}
         (VV^{T})_{ij} = \begin{cases}
            1 & \text{ if } i = j \\
            \leq \mu & \text{ if } i \neq j
        \end{cases}
   \end{align}
\end{lemma}

Low-rank factorizations of the identity matrix arise in the context of Johnson-Lindenstrauss lemma (see \cite{alon_2009}). 
They can be explicitly constructed using codewords from a binary code of distance $1/2 - \mu$ and rate $\Omega(1/\mu^2)$.

Given a $n \times k$ matrix $V$ with the above properties, it is clear that the columns of the matrix $V$ approximately span the rows of the identity matrix.  In fact, the columns of $V$ form a basis of size $k = O(\log n/\mu^2)$ that yield a $\mu$-approximation to each of the $n$ rows of the identity matrix.  We will use these vectors to construct a low rank approximation to the interval functions.

\begin{definition}[The basis $\hat{W}(\mu, m)$] \label{def:basis_lcvx}
Fix $\mu \in \intzo$ and $m = 2^{r}$ for some $r \in \N$. 
For each $h \in \{ 0, \ldots, \log m\}$, let $V^{(h)}$ denote $\frac{m}{2^h} \times k_h$ the matrix given by \cref{lem:codes} with $n = m/2^h$ and $\mu$.  
Let $v^{(h)}_1,\ldots, v^{(h)}_{k_h} \in \R^{m/2^h}$ denote the columns of the matrix $V^{(h)}$.
For $k \in [k_h]$, define the  functions $\hat{w}_k^{h}:[m] \to \intpmo$, by setting
\begin{align}
    \label{eq:def-basis}
    \hat{w}_k^h(y) = v^h_{k}(a) \ \ \ \ \forall  a \in [m/2^h]  \text{ and } y \in [2^{h} \cdot a,2^{h} \cdot (a+1)).
\end{align} 
In other words, the vector $\hat{w}_k^{h}$ is obtained from $v^{h}_k$ by repeating each element $2^h$ times consecutively. 

Let \[ \hat{W}^{(h)}(\mu,m) = \{ \hat{w}_k^h \}_{k \in [k_h]}\] 
and let 
\[\hat{W}(\mu, m) = \bigcup_{h = 0}^{\log m} W^{(h)}(\mu,m)\]

\begin{lemma}[Approximating intervals]\label{lemma:approx-intervals}
For every $h = \{0,\ldots, \log m\}$, the vectors $\hat{W}^{(h)}(\mu,m)$ $\mu$-approximately span the dyadic intervals $\II_{a \cdot 2^h, (a+1)2^h-1}$ for all $a \in [m/2^h]$.

Therefore, their union $\hat{W}(\mu,m)$ $\mu$-approximately spans all dyadic intervals in $[m]$.
\end{lemma}
\begin{proof}
First, consider the case $h = 0$.  In this case, the entries of the matrix $V^{(h)} (V^{(h)})^T$ approximate the entries of the identity matrix within an error $\mu$.  Hence the columns of the matrix $V^{(h)}$ $\mu$ approximately span the rows of the identity matrix.  Note that the rows of the identity matrix are the dyadic interval functions $\II_{a,a}$ for $a \in [m]$.

By the same argument, for any $h \in \{1,\ldots, \log{m}\}$, the columns of $V^{(h)}$ $\mu$-approximately span the rows of the identity matrix of dimension $m/2^h$.
Note that the entries of vectors $\hat{w}_k^h$ are obtained by repeating each entry of $v^h_k$ $2^{h}$ times consecutively. For any row of the identity matrix of dimension $m/2^h$, repeating its entries $2^h$ times consecutively will yield the indicator function of a dyadic interval of length $2^h$.
Hence it follows that the dyadic interval functions of length $2^h$ are $\mu$-approximately spanned by $\hat{W}^{(h)}(\mu,m)$. 
\end{proof}

\end{definition}

\paragraph{Putting things together}

We now have all the ingredients to prove the upper bound of \cref{thm:cvx-approximation}.
\begin{proof}[Proof of upper bound from \cref{thm:cvx-approximation}]
Fix $m = \frac{1}{\delta}$, $t = m^{1/3}$ and $\mu = \frac{1}{12 m^{1/3} \log m}$.  Using \cref{lemma:approx-intervals}, we get that $\hat{W}(\mu,m)$ is a $\mu$-approximate basis for all dyadic interval functions $\mathbb{D}^{m}$.

By \cref{prop:dyadic}, every interval in $\II^{m}$ is a union of at most $2 \log m$ dyadic intervals, and thus, by the triangle inequality, $\hat{W}(\mu,m)$ is a $(2 \mu \log m)$-approximately basis for all intervals $\II^{(m)}$.

Now we appeal to \cref{prop:build-basis} with $t = m^{1/3}$. Thereby, we conclude that the set of functions $\hat{W}(\mu,m) \cup \{\relu_{it} \;|\; i \in [m^{2/3}]\}$ is an approximate basis for all $\relu$ functions with an error of $t \cdot (2 \mu \log m) \leq 1/6$.

By \cref{cor:relu}, the same basis approximates all functions in $\Cvxvec$ with an error of $3 \cdot 1/6 = 1/2$. 

Finally, using \cref{lem:cvx-discretization}, this yields a corresponding basis of functions over $[0,1]$ that is a $\delta \cdot \frac{1}{2} + \delta = 3\delta/2$-approximation for $\Cvx$.

The size of the family $\hat{W}(\mu,m)$ is given by,
\begin{align}
|\hat{W}(\mu,m)| = \sum_{h = 0}^{\log m} |\hat{W}^{(h)}(\mu,m)| = \sum_{h = 0}^{\log m} k_h \leq O\left(\log m \cdot \frac{\log m}{\mu^2} \right) = O(m^{2/3} \log^3 m)
\end{align}
and thus the total size of the basis \eat{$\hat{W}(\mu,m) \cup \{ \relu_{it} \;|\; i \in [m/t]\}$ }is $O(m^{2/3} \log^3 m) + O(m/t) = O(m^{2/3} \log^{3} m)$.

This completes the proof of the upper bound \cref{thm:cvx-approximation} (after substituting $\delta$ for $2\delta/3$, so that we have a $\delta$-approximation in the end).
\end{proof}

\subsection{Lower bounds for \texorpdfstring{$\delta$}{delta}-approximate dimension}

The main goal of this section will be to prove the lower bound of \cref{thm:cvx-approximation}, i.e., that any $\delta$-approximate basis for $\Cvx$ must have size $\Omega(1/\delta^{2/3})$.

First, as an aside, we show that approximating all Lipschitz functions on $[0, 1]$ actually requires a basis of size $\Omega(1/\delta)$. (This is tight by the remarks after \cref{lem:discretization}.) Thus, restricting ourselves to convex functions actually gives an advantage in the $\eps$-approximate dimension, reducing it from $\Theta(1/\delta)$ to $\tilde \Theta(1/\delta^{2/3})$. Specifically, let $\Lip$ denote the set of 1-Lipschitz functions on $[0, 1]$. Then, we have the following result.

\begin{theorem}
    \label{thm:lip-lower}
     For all $\delta \geq 0$, 
        $\dim_{\delta}(\Lip) \geq \lfloor 1/4\delta \rfloor $.    
\end{theorem}
\begin{proof}
Let $G_{4\delta} = \{4 \delta t \;|\; t = 0,\ldots, \lfloor 1/4\delta \rfloor \}$ denote a grid over $[0,1]$ of separation $4 \delta$. Observe that any function $f: G_{4\delta} \to \{-2\delta,2\delta\}$ can be extended to a $1$-Lipschitz function over $[0,1]$.
Indeed, one can pick a piecewise linear function that coincides with $f$ on $G_{4 \delta}$, and is linear in all intermediate intervals. 
    
Now assume for contradiction that there exists $\Scal$ such that 
$\dim(\Scal) < \lfloor 1/4\delta \rfloor$ which gives $\delta$-uniform approximations to $\Lip$.
Since $|G_{4 \delta}| = \lfloor 1/4 \delta\rfloor + 1$, by dimension counting, there exists some function on $G_{4\delta}$ which is orthogonal the restriction of every function in $\Scal$ to $G_{4\delta}$.
Specifically, there exists nonzero $g: G_{4\delta} \to \R$ which is non-zero, such that for every $s \in \Scal$ (and thus for every $s \in \mathrm{Span}(\Scal)$),
\[ \sum_{y \in G_{4\delta}} g(y)s(y) = 0\]
Consider the function $f \in \mF$ where
\[ f(y) = \sign(g(y))\cdot 2\delta \ \ \forall y \in G_{4\delta}.\]

The function $f$ admits a $1$-Lipschitz extension to $[0,1]$, and therefore can be $\delta$-approximated by functions in $\mathrm{Span}(\Scal)$. That is, there exists $s \in \mathrm{Span}(\Scal)$ which is a $\delta$-approximation to $f$. But this means that 
\[ \sign(s(y)) = \sign(f(y)) = \sign(g(y)),\] 
and $|s(y)| \geq \delta$ for $y \in G_{4\delta}$. But then $g$ cannot be orthogonal to $s$, giving a contradiction. 
\end{proof}

To prove a lower bound on the $\eps$-approximate dimension of $\relu$, we use Alon's lower bound on the approximate rank of the identity matrix. 

\begin{theorem}[\cite{alon_2009}]
    \label{thm:alon} 
    Let $I_n$ be the $n \times n$ identity matrix, and let $1/(2\sqrt{n}) \leq \mu \leq 1/4$. Then,
    \[ \rank_\mu(I_n) \geq \frac{d\log(n)}{\mu^2\log(1/\mu)},\]
    for some absolute constant $d$. (Here, $\rank_\mu(I_n)$, the $\eps$-approximate rank of $I_n$, denotes the $\eps$-approximate dimension of its rows.)

\end{theorem}

\begin{theorem}
For any absolute constant $c > 0$ and all $m \in \Z^+$,
\[\dim_{c}(\relu_{[m]}) \geq \Omega(m^{2/3}),\]
where $\relu_{[m]}$ denotes the family $\relu_{[m]} = \{\relu_i \;|\; i \in [m]\}$ of functions on $[m]$.
\end{theorem}

\begin{proof}
Suppose that $\Scal$ is a $c$-approximate basis for $\relu_{[m]}$.

Fix $t = m^{1/3}$, and let $A = \{1,2,\ldots, m/t - 1\}$.  Let $\Scal'$ denote the functions in $\Scal$ restricted to the domain $t \cdot A = \{ t \cdot i | i \in A\}$.  Clearly, $|\Scal'| \leq |\Scal|$.

For any $i \in A$, consider the function $f_i:A \to \R$ defined as
\begin{align}
f_i(y) = \frac{1}{t} \big(\relu_{(i+1)t}(yt)+\relu_{(i-1)t}(yt) - 2\relu_{it}(yt)\big).
\end{align}

By substituting the values of $y$, it is easy to check that \begin{align}
f_i(y) = \begin{cases}
            1 & \text{ if } y = i\\
            0 & \text{ if } y \neq i
            \end{cases}
\end{align}
In other words, the functions $f_i$ are the rows of the identity matrix of dimension $|A|$.

However, if the basis $\Scal$ yields an $c$-approximation for each of the three functions $\relu_{(i+1)t}, \relu_{it}$ and $\relu_{(i-1)t}$, then $\Scal'$ yields a $4c/t$-approximation for the functions $f_i$. 

By appealing to Alon's lower bound on the approximate rank of the identity matrix, we get that 
\[ |\Scal'| \geq d \frac{\log |A|}{(4c/t)^2 \log (t/4c)} \geq \Omega(m^{2/3}),\] as desired. 
\end{proof}

From the above, we immediately conclude a lower bound on the $\delta$-approximate dimension of $\Cvx$:

\begin{corollary}
For all $\delta > 0$, 
\[ \dim_{\delta}(\Cvx) \geq \Omega\left(\frac{1}{\delta^{2/3}}\right),\]
where $\Omega$ hides an absolute constant factor.
\end{corollary}
\begin{proof}
Fix $m = \lfloor 1/\delta \rfloor$.  Given a $\delta$-approximation to $\Cvx$, we get a $\Omega(1)$-approximation to $\relu_{[m]}$ by considering the $\relu$ functions over $[0,1]$, and restricting to the evaluation points $\{i/m \;|\; i \in [m]\}$.
\end{proof}

Since $\relu$ functions are just linear transformations of $L_1$ loss functions, this also similarly implies that the $\delta$-approximate dimension of $L_1$ loss functions is large:
\begin{corollary}
Let $\mathcal{L}_1$ denote the set of $L_1$ loss functions of the form $|y-t|$ for $t \in [0, 1]$. Then, for all $\delta > 0$, 
\[ \dim_{\delta}(\mathcal{L}_1) \geq \Omega\left(\frac{1}{\delta^{2/3}}\right).\]
\end{corollary}

%% file: sections/calibration.tex
\section{Loss Minimization}
\label{sec:sufficient}

\subsection{Loss Functions, Sufficient Statistics and Uniform Approximations}

A loss function is a function $\ell: \Y \times \Acal \to \R$ which assigns a real value $\ell(y, t)$ to a pair of inputs where $y \in \Ycal$ is a label and $t \in \Acal$ is an action. We will focus on the case where $\Y \subseteq  \intzo$. We will allow for arbitrary actions sets $\Acal$. They can be discrete (e.g buy or not buy), or continuous (e.g in some bounded interval [-B, B]).\footnote{Our notion of an abstract space of actions departs from some prior work on omniprediction \cite{gopalan2021omnipredictors, gopalan2022loss} which required the set of actions to be a bounded subset of $\R$, bringing it in line with the calibration literature (for instance \cite{ucal}).}

Let $\llip$ denote the set of all $\ell$ that are $1$-Lipschitz in $y$ for every $t \in \Acal$. Let $\lcvx \subseteq \llip$ denote the subset of functions where $\ell(y,t)$ is convex in $y$ for every $t \in \Acal$. 

We define a family of statistics to be a set of functions $\Scal = \{s_i: \Ycal \to [-1,1]\}_{i=0}^d$, with the convention that $s_0 = 1$ is always the constant function.

\begin{definition}[$(d, \lambda, \delta)$-uniform approximations, sufficient statistics]
\label{def:approx-bi}
    Let  $\Scal$ be family of statistics and $\Lcal$ be a family of loss functions. We say that $\Scal$ gives $(d, \lambda, \delta)$-uniform approximations to $\Lcal$ where $d =|\Scal|$ if for every $\ell \in \Lcal$, there exist $\{r^\ell_i: \Acal \to \R\}_{i =0}^d$ such that
    \[ \abs{\sum_{s_i \in \Scal  } r^\ell_i(t) s_i(y) - \ell(y,t)} \leq \delta\]
    and
    \[ \sum_{i = 0}^d|r^\ell_i(t)| \leq \lambda\]  
    
    for all $y \in \Ycal$ and for all $t \in \Acal$.
    Equivalently, $\Scal$ $\epsilon$-approximately spans the family $ \{ \ell_t(y) = \ell(y,t) \}$ for every $t \in \Acal$ with coefficients of total magnitude at most $\lambda$.
    We refer to $\Scal$ as a set of sufficient statistics for $\Lcal$, and to
    $\Rcal = \{r^\ell_i\}_{i \in [d], \ell \in \Lcal}$
    as the coefficient family. 
    
 \end{definition}

In light of the above definition, it is useful to define 
the function $\hell:[-1,1]^d \times \Acal \to \R$ as
    \begin{align}
        \label{eq:lhat}
            \hell(v, t) =  r_0^{\ell}(t) +  \sum_{i=1}^d r^\ell_i(t)v_i
    \end{align}
    which acts on the statistics from $\Scal$ directly instead of on $y$. 
     Note that for a distribution $\mD$ on $\Ycal$, for $\ell \in \Lcal$ and $t \in \Acal$, \cref{def:approx-bi} implies
    \begin{align}
    \label{eq:approx-stats}
          \hell\left( \E_\mD[s_i(\y)], t \right) =r^\ell_0(t) + \sum_{i=1}^d r^\ell_i(t)\E_\mD[s_i(\y)]  =  \E_\mD[\ell(\y,t)]  \pm \delta  .
    \end{align}
    Hence the expectations of functions in $\Scal$ gives a set of statistics that lets us approximate the loss associated with each action in $\Acal$ for every loss in $\Lcal$, thus  justifying the term \emph{sufficient statistics}.

For sake of intuition, we present the example $\Lcal^C_{\zo}$ of all bounded loss functions $\ell:\zo \times \Acal \to [-C,C]$ in the case of $\Ycal =\zo$ i.e. binary classification.

\begin{proposition}
    For all $C > 0$,  $\Scal =\{1, y\}$ gives $(1, 2C, 0)$ uniform approximations to $\Lcal^C_{\zo}$.
\end{proposition}
\begin{proof}
    We can write $\ell(y,t)$ using its multilinear expansion in $y$ as
    \[ \ell(y, t) = \ell(0, t) + y(\ell(1,t) -\ell(0,t)).\]
    We take 
    \[ s_1(y) = y, \ r^\ell_0(t) = \ell(0,t), \ r^\ell_1(t) = \ell(1,t) -\ell(0,t). \]
    It follows that $\lambda = \max_t |r^\ell_1(t)| \leq 2C$.
\end{proof}

We record the following corollary that allows us to assume that $\lambda = O(d)$ in a family of sufficient statistics. 
The fact is standard in convex geometry and follows from a simple application of John's theorem; we include a proof for completeness in Appendix \ref{app:proofs}.

\begin{corollary} \label{cor:bound_coeff}
    Let $ \Lcal $ be a family of loss functions bounded by $C$, with a family of sufficient statistics $\Scal$ that gives a $(d, \lambda , \delta)$-approximation to $\Lcal$ for some $\lambda$. 
    Then, there exists a family of statistics $ \Scal'$ consisting of functions also bounded by $C$ which gives a $(d, (1+\delta/C) d , \delta)$-approximation to $\Lcal$.    
\end{corollary}

\subsection{Statistics, predictors, and calibration}

Next, we will define the notion of a predictor corresponding to a family of statistics. 
Let $\mD^*$ denote a distribution on $\X \times \Ycal$. 
We denote samples from $\mD^*$ by $(\x, \y^*)$.
Given a family of statistics $\Scal = \{s_i: \Y \to [-1,1]\}_{i \in [d]}$, let $s(y) = (s_i(y))_{i \in [d]}$. An $\Scal$-predictor is a function $p: \X \to [-1,1]^d$ with the interpretation that $p(x)$ is an estimate for $\E[s(\y)|\x =x]$. 
As an example, consider $s_i(y) = y^i$. 
Predictors for this family would predict the first $d$ moments of $\y|x$ for each $x$.

We now define the notion of a calibrated predictor of statistics.

\begin{definition} \label{def:calibration}
    Let $\Scal$ be a family of statistics.
    We say the predictor $p$ is $\beta$-calibrated for $\Scal$ under $\mD^*$ if
\begin{align}  
    \E_{\mD^*}\lt[\norm{\E[s(\y^*)|p(\x)] - p(\x)}_\infty \rt] \leq \beta.
\end{align}
Perfect calibration is said to holds when $\E_{\mD^*}[s(\y^*)|p(\x)]  = p(\x)$ i.e. $\beta =0$.
\footnote{We could use $\ell_1$ or other $\ell_p$ norms in place of $\ell_\infty$, all such definitions are equivalent up to polynomials in $d$.}
\end{definition}

Let $\Im(\Scal) \subseteq [-1,1]^d$ denote the set of values $\E[s(\y)]$ can take over all distributions on $\y$. 
This set is generally a proper subset of $[-1,1]^d$, due to relationships between functions in $\Scal$.
For instance when $\Scal$ is the set of the first $d$ moments, it needs to satisfy various moment inequalities. 
Perfect calibration ensures that every prediction lies in $\Im(\Scal)$. 
We will next show a robust analogue of this, which shows that 
$\beta$-calibration implies our predictions are close to the expectations of $s_i$s for a suitably defined distribution on labels.
In order to do this, we define the following {\em simulated distribution} corresponding to a predictor. 

 \begin{definition}[Simulated distribution] \label{def:tilde-dist}
    Let $\Dcal^*$ be a distribution on $\Xcal \times \Ycal$ and let $p$ be a $\Scal$-predictor for a statistics family $\Scal$.
     We will associate a distribution $\tmD = \tmD(p)$ on points and labels to $p$ sampled as $(\x, \ty) \sim \tmD$, we first sample $\x \sim \mD^*$ and let $\ty|\x \sim \y^*|p(\x)$. 
 \end{definition}

 From above definition, the marginal distribution of $\x$ matches $\mD^*$, whereas $\ty$ is identically distributed over each level set of $p$. This lets us couple the distributions $\mD^*$ and $\tmD$.
 We sample $\x$ according to the common marginal, and then sample $\y^*|\x \sim \mD^*$ and $\ty|\x \sim \tmD$ independently. 
 
  We note that our definiton of the simulation is different from  Boolean setting \cite{OI, gopalan2022loss, bernoulli}, where the simulation is based on the predictor alone, and is independent of the distribution $\mD^*$ that is being learnt. It is reminiscent of the view of \cite{gopalan2021omnipredictors}, who view predictors as partitions of the space into level sets, and define a {\em canonical prediction} which is the expectation over the level set. Our next lemma may be viewed as showing the closeness of a calibrated predictor to precisely such a canonical predictor. 

\begin{lemma}
\label{lem:cal-dist}
    If $p$ is $\beta$-calibrated under $\mD^*$,  
    \[ \E_{\tmD}\lt[\norm{p(\x) -\E[s(\ty)|\x]}_\infty\rt] \leq \beta.\] 
\end{lemma}
\begin{proof}
    From the definition of $\ty$, it follows that for $s_i \in \Scal$, $
    \E[s_i(\ty)|\x] = \E[s_i(\y^*)|p(\x)]$.
    Hence,
    \begin{align}
        \E[\norm{\E[s(\ty)|\x] - p(\x)}_\infty] &= \E[\norm{\E[s(\y^*)|p(\x)] - p(\x)}_\infty] \leq \beta. 
    \end{align}
    where the inequality is by Definition \ref{def:calibration}. 
\end{proof}

\eat{
The next lemma gives a robust analogue of this statement.

\begin{lemma}
    For an $s$-predictor $p$,  $\E_{\mD}[d(p(\x), \Im(s))] \leq \CE_\mD(p)$. 
\end{lemma}
\begin{proof}
    We have
    \begin{align}
        \CE_\mD(p) = \E\lt[\norm{\E[s(\y)|p(\x)] - p(\x)}_1\rt] \leq \E[d(p(\x), \Im(s))]
    \end{align}
    since $\E[s(\y)|p(\x)] \in \Im(s)$.
\end{proof}
}
\eat{We can explicitly 
Given an $R$-predictor $p:\X \to [-1,1]^d$ and a distribution $\mD$ on $\X \times \intzo$, define the distribution $\tilde{\mD}$ on $\X \times \intzo$ as follows. We sample $(\x, \y) \sim \mD$. We define $\tilde{\y}$ to be $\y|p(\x)$. We then return $(\x, \ty)$ as a random sample from $\tilde{\mD}$. In effect, we resample $\ty$ according to $\mD|p(\x)$.} 

\subsection{Optimal Actions under Loss Functions}

    Given a distribution $\Dcal$ and a loss function $\ell$, we can define the {\em optimal} action 
    
    as
    \[ k_\ell(\Dcal) = \arg\min_{t \in \Acal} \E_{\y \sim \Dcal} [\ell(\y, t)].\]  
    
    As defined, the function $k_\ell$ requires full knowledge of the distribution $\mD$. We will see that if 
     if $\Scal$ is a set of sufficient statistics for $\ell$, then one can approximate $k_\ell$ with just the knowledge of $\E[s(\y)]$.

    Assume that $ \Scal$ gives $(d, \lambda, \eps)$-uniform approximations to $\ell$, so that $\ell$ is $\eps$-approximated by $\hat{\ell}$. Selecting action via $\hkl$ results in actions that are at most $O(\eps)$ far from optimal for $\ell$.
    But
    \begin{align} 
        \hkl(\mD) =  \arg\min_{t \in \Acal} \E_{\mD}[\hell(s(\y), t)]
                  =  \arg\min_{t \in \Acal} \left[ r^\ell_0(t) + \sum_{i=1}^d r^\ell_i(t)\E_{\mD}[s_i(\y)] \right],
    \end{align}
    so $\hkl$ only depends on  $\E_\mD[s(\y)]$ rather than the entire distribution $\mD$. 
    
    Abusing notation, we extend the definition of $\hkl$ to take predictions in  $[-1,1]^d$ as its argument. That is, define $\hkl:[-1,1]^d \to \Acal$ as
    \begin{align}
        \hkl(v) = \arg\min_{t \in \Acal} r^{\ell}_0(t) + \sum_{i=1}^d r^\ell_i(t)v_i . 
    \end{align}
    Note that, for $v  = \E[s(\y)] \in \Im(\Scal)$, this matches our prior definition, since $\hkl(\E[s(\y)]) = \hkl(\mD)$. But it also allows for general $v \not\in \Im(\Scal)$. This will be important since our $\Scal$-predictors are not guaranteed to make predictions in $\Im(\Scal)$.

    We do not impose any constraints on the action space $\Acal$, or how the loss family $\mL$ depends on it. We only require the existence of an oracle for $\Acal$ that solves the minimization problem required for computing $\hkl$. In the case of a discrete set of actions, this can simply be done by enumeration. In the case when $\Acal$ is a compact set such as a bounded interval, and the loss functions $\ell$ to be Lipschitz in the actions, we could discretize the action space and compute the value at each choice of the discretization, to find an approximate minimum. Our reason for abstracting away the complexity of computing $\hkl$ is that even if we learnt the Bayes optimal $\Scal$-predictor, we would still need to compute $\hkl$, so the complexity of this function is extraneous to the task of learning a good predictor.

\subsection{Multiaccuracy}

Finally, we define the notion of multiaccuracy with respect to a class of tests $\Bcal = \{b: \X \to \R\}$. The notion was defined in the context of binary classification by \cite{hebert2018multicalibration}, though similar notions have appeared previously in the literature on boosting and learning. \footnote{In previous work, multiaccuracy was defined with respect to a hypothesis class $\mC$, which mapped $\X$ to $\R$. Since we define hypotheses classes to map to $\Acal$, we use the term tests for functions mapping to $\R$. The specific tests we use will compose $c \in \mC$ with a function $r_\ell:\Acal \to \R$.}

We extend the definition of multiaccuracy to the setting of statistics prediction (similar to \cite{JLPRV21, bernoulli}).  Intuitively, multiaccuracy for a predictor $p$ for a family of statistics $\Scal$ requires that no test $b$ in the class $\Bcal$ can distinguish the true value of a statistic $s_i \in \Scal$ from the predicted value $p_i$.

\begin{definition}[Multiaccuracy]
    Let $\Scal$ be a family of statistics, $\mB = \{b : \X \to \R\}$ be a class os tests and $\alpha > 0$. 
    We say that an $\Scal$-predictor $p:\X \to [-1,1]^d$ is $(\mB, \alpha)$-multiaccurate if for every $i \in [d]$ and $b \in \mB$, it holds that
    \[ \abs{\E_{\mD^*}[(s_i(\y^*) - p_i(\x))b(\x)]} \leq \alpha. \]
\end{definition}

%% file: sections/omni.tex
\section{Omniprediction via outcome indisinguishability} 
\label{sec:omni_Oi}

An omnipredictor, introduced in the work of \cite{gopalan2021omnipredictors}, is a predictor can be postprocessed to get an action that suffers lesser loss than any hypothesis in the class $\Ccal$. The original definition was in the setting of binary or multiclass classification, where the predictor returns a probability distribution on labels. The following definition generalizes this notion to $\Scal$-predictors. 

\begin{definition}[Omnipredictor \cite{gopalan2021omnipredictors}]
    Let $\Lcal$ be a family of loss functions and $\Ccal$ be family of hypotheses, and $\eps > 0$. Let $\Scal$ be a set of sufficient statistics for $\Lcal$. An $\Scal$-predictor $p: \X \to [-1,1]^d$ is an $(\Lcal, \Ccal, \eps)$-omnipredictor if for every $\ell \in \Lcal$, there exists $k:[-1,1]^d \to \Acal$ such that 
    \[ \E_{\mD^*}[ \ell(\y^*, k(p(\x)))] \leq \min_{c \in \Ccal} \E_{\mD^*}[\ell(\y^*, c(\x)] + \eps. \]
\end{definition}
Note that the set of sufficient statistics $\Scal$ needs to give a $(d, \lambda, \delta)$-uniform approximations to $\Lcal$ with $\delta \leq \varepsilon$. Recall that $\Scal$ is associated with a coefficient family $\Rcal = \{r^\ell_i\}_{i \in[d], \ell \in \mL}$. We let $\Rcal \circ \mC$ denote all functions of the form $r^\ell_i \circ c$ where $r^\ell_i \in \Rcal$ and $c \in \mC$. 

\begin{remark}
    \label{rem:omni}
    Note that $ \Rcal \circ \Ccal $ only considers composition of $c$ with $r_i$ for $i > 0 $. 
    In particular, it does not consider compositions of $c$ with $r_0$.
    For our main result, it will suffice to not consider compositions of $c$ with $r_0$. 
\end{remark}

Our main result in this section establishes sufficient conditions for omniprediction. It shows that for any family of loss functions that can be well-approximated by a family of statistics, we can get an omnipredictor through calibration and multiaccuracy.

\begin{theorem}
\label{thm:main}
    Let  $\Scal$ be family of statistics and $\Lcal$ be a family of loss functions such that that $\Scal$ gives $(d, \lambda, \delta)$-uniform approximations to $\Lcal$ with coefficient family $\Rcal$. If the $\Scal$-predictor $p$ is $(\Rcal \circ \mC, \alpha)$-multiaccurate and $\beta$-calibrated then it is an $(\Lcal, \Ccal, \eps)$-omnipredictor for 
    \[ \eps = 3(d\alpha + \lambda \beta + \delta) \]
    using the functions $\hkl$ for choosing actions. 
\end{theorem}

Following the loss outcome indistinguishability paradigm of \cite{gopalan2022loss}, we will prove this result by showing:
\begin{itemize}
    \item The predictor $p(\x)$ is an omnipredictor for the distribution $\tmD$ \cref{def:tilde-dist} and the family of losses $\hat{\Lcal} = \{\hell\}_{\ell \in \Lcal}$ where we $\hkl(p(\x))$ to choose actions. 
    
    \item One can switch the label distribution from $\ty$ to $\y^*$ and the losses from $\hell$ to $\ell$ without much change in the expected loss. 
\end{itemize}
To implement this, we show a sequence of lemmas showing various forms of indistinguishability for labels and loss functions. The first shows indistinguishability for expected loss $\hell$ when the actions are functions of the prediction $p(\x)$. 

\begin{lemma}
    \label{lem:1}
    For all functions $k:[-1,1]^d \to \Acal$    
    \begin{align}
        \E_{\mD^*}[\hell(s(\y^*), k(p(\x)))] = \E_{\tmD}[\hell(s(\ty), k(p(\x)))].
    \end{align}
\end{lemma}
\begin{proof}
We can write
\begin{align}
   \E_{\mD^*}[\hell(s(\y^*), k(p(\x)))] &=  \E\lt[\sum_{i=0}^d r^\ell_i(k(p(\x)))s_i(\y^*)\rt]\\
   &= \E\lt[\sum_{i=0}^d\E[r^\ell_i(k(p(\x)))s_i(\y^*)|p(\x)]\rt]\\
   &= \E\lt[\sum_{i= 0}^d \E[r^\ell_i(k(p(\x)))s_i(\ty)|\x]\rt]\\
   &= \E\lt[\sum_{i=0}^d r^\ell_i(k(p(\x)))s_i(\ty)\rt]\\
   &= \E_{\tmD}[\hell(s(\ty), k(p(\x))].
\end{align}
\end{proof}

The next lemma shows that if $p$ is well-calibrated, then distinguishing the predictions $p(\x)$ from $\E[s(\ty)|x]$ using $\hell$ is hard, even allowing 
for arbitrary actions.

\begin{lemma}
\label{lem:2}
    If $p$ is $\beta$-calibrated, then for all functions $b: \X \to \Acal$
    \[ \abs{\E[\hell(s(\ty), b(\x))] - \E[\hell(p(\x), b(\x))]} \leq \lambda \beta. \] 
\end{lemma}
\begin{proof}
    For any function $b: \X \to \Acal$, we have
\begin{align}
\label{eq:switch-t}
    \abs{\E_{\tmD}[\hell(s(\ty), b(\x))] - \hell(p(\x), b(\x))} & = \abs{\E\lt[ \sum_{i=1}^d r^\ell_i(b(\x))(p_i(\x) - s_i(\ty))\rt]}\\
    & = \abs{\E\lt[ \sum_{i=1}^d r^\ell_i(b(\x))(p_i(\x) - \E[s_i(\ty)|\x])\rt]}\\
    & \leq \E\lt[\lt(\sum_{i=1}^d \abs{r^\ell_i(b(\x))}\rt) \max_{i \in [d]} \abs{p_i(\x) - \E[s_i(\ty)|\x]}\rt] \tag{Holder's inequality}\\
        & \leq  \E[\lambda\norm{p(\x) -\E[s(\ty)|\x]}_\infty]\\
        & \leq \lambda \beta.
\end{align}
\end{proof}

Next we show more general conditions under which $\ell$ does cannot distinguish $\ty$ from $\y^*$. \cref{lem:1} gave such a result but for limited actions. Here we allow more general action functions, but we also make more assumptions about the predictor.
\begin{corollary}
\label{cor:cma}
    If  $p$ is $(\Rcal \circ \mC, \alpha)$-multiaccurate and $\beta$-calibrated then
    \begin{align}
        \abs{\E_{\tmD}[\hell(s(\ty), c(\x))] - \E_{\mD^*}[\hell(s(\y^*), c(\x))]} \leq d\alpha + \lambda \beta. 
    \end{align}
\end{corollary}
    \begin{proof}
     We can write    
     \begin{align}
          \abs{\E_{\tmD}[\hell(s(\ty), c(\x))] - \E_{\mD^*}[\hell(s(\y^*), c(\x))]} & \leq \abs{\E[\hell(s(\ty), c(\x))] - \E_{\mD}[\hell(p(\x), c(\x))]} \\
          & + \abs{\E[\hell(s(\y^*), c(\x))] - \E[\hell(p(\x), c(\x))]}.
    \end{align} 
    The first term can be bounded by $\lambda \beta$ using $\beta$-calibration together with \cref{lem:2}.
     To bound the second term, we note that 
     \begin{align}
        \abs{\E[\hell(s(\y^*), c(\x))] - \E_{\mD}[\hell(p(\x), c(\x))]}
          &= \abs{\E\lt[ \sum_{i=1}^d r^\ell_i(c(\x))(s_i(\y^*) -  p_i(\x)\rt]}\\
          & \leq \sum_{i=1}^d \abs{\E[r^\ell_i(c(\x))(s_i(\y^*) -  p_i(\x))]} \\
          & \leq d\alpha 
     \end{align}
     where we use multiaccuracy for each $i$. 
    \end{proof}

We now complete the proof of \cref{thm:main}, our main result on omniprediction.

\begin{proof}[Proof of \cref{thm:main}]
 We have the following chain of inequalities
    \begin{align}
        \E[\hell(s(\ty), \hkl(p(\x)))] & \leq \E[\hell(p(\x), \hkl(p(\x)))] +  \lambda \beta \tag{By \cref{lem:2} }\\
        & \leq \E[\hell(p(\x), c(\x))]  + \lambda \beta\tag{by definition of $\hkl$}\\ 
        & \leq \E[\hell(s(\ty), c(\x))  +  2\lambda \beta  \label{eq:ineq1}.
    \end{align}
To switch the label distribution from from $\ty$ to $\y^*$ we use
\begin{align}
    \E[\hell(s(\y^*), \hkl(p(\x)))] &= \E[\hell(s(\ty), \hkl(p(\x)))] \tag{By \cref{lem:1}}\\
    \E[\hell(s(\ty), c(\x)) & \leq \E[\hell(s(\y^*), c(x))] + (d\alpha + \lambda \beta) \tag{\cref{cor:cma}}
\end{align}
    Plugging these into Equation \eqref{eq:ineq1} gives
\begin{align}
    \E[\hell(s(\y^*), \hkl(p(\x)))] \leq \E[\hell(s(\y^*), c(\x))] + d\alpha + 3\lambda\beta. \label{eq:ineq2}
\end{align}
Finally we can switch each loss from $\hell$ to $\ell$ by incurring an additional $\delta$. We use the uniform approximation property with $\hkl(p(\x))$ and $c(\x)$, which gives
\begin{align}
    \abs{\E[\ell(\y^*, \hkl(p(\x)))] - \E[\hell(s(\y^*), \hkl(p(\x)))]} \leq \delta,\\
    \abs{\E[\ell(\y^*, c(\x))] - \E[\hell(s(\y^*), c(\x))]} \leq \delta
\end{align}
Plugging these into Equation \eqref{eq:ineq2} gives the desired bound.    
\end{proof}

%% file: sections/low_degree.tex
\section{Main Applications}
\label{sec:omni_app}

In this section, we will derive omnipredictors for various classes of loss functions.  
First, in \cref{sec:omni_cvx}, we will present an omnipredictor for the class of convex, Lipschitz loss functions. 
In \cref{sec:omni_moments}, we will present an omnipredictor for the class of functions approximated by low degree polynomials, in particular, for the class of $\ell_p$ losses. 
In \cref{sec:omni_GLM}, we present an omnipredictor for the class of losses corresponding to generalized linear models.

As mentioned earlier, the main idea is to find a family of sufficient statistics that approximates a family of loss functions. 
Given the approximations, the main algorithmic result driving the omnipredictors for various classes is the following theorem below. This theorem bounds the sample complexity of achieving a predictor that satisfies calibrated multiaccuracy with respect to family of statistics and a family of test functions (corresponding to the composition of the coefficient family in the approximation of the losses and the comparison class of hypotheses).
We state the theorem here and defer the proof and further discussion to \cref{sec:omni_alg}.

\begin{restatable}{theorem}{mainAlg} \label{thm:main_alg}
    Let  $\Scal$ be family of statistics and $\Lcal$ be a family of loss functions such that that $\Scal$ gives $(d, \lambda, \eps)$-uniform approximations to $\Lcal$ with coefficient family $\Rcal$, there exists an algorithm that returns an $(\Lcal, \Ccal, 4\eps)$-omnipredictor\footnote{Note that once $\Scal$ is fixed, the $3\eps$ factor in the omniprediction slack is unavoidable.}, satisfying the following properties.
    \begin{itemize}
        \item The algorithm makes $O(d/\sigma^2)$ calls to a $(\rho, \sigma)$-weak learner for $\Rcal \circ \Ccal$ where $\sigma \leq \rho \leq \eps/12\lambda$. 
        
        \item The algorithm has time and sample complexity $\tilde{O} \left( d \lt(\frac{\lambda}{\eps} \rt)^{d+5} +  \frac{d}{\sigma^2} Z\right)$ 
        where $Z$ is the runtime/sample complexity of the weak learner.
    \end{itemize}
\end{restatable}

\input{sections/omni_convex_lipshitz.tex}

\input{sections/moments.tex}

\input{sections/bregman}

%% file: sections/omni_convex_lipshitz.tex
\subsection{Omniprediction for Convex Lipschitz Losses} \label{sec:omni_cvx}

Recall that $\lcvx$ denotes the family of convex, Lipschitz loss functions i.e. loss functions $ \ell(y,t) $ that are Lipschitz and convex in $y$.    
For this class, we can derive approximations in terms of a small sized family of statistics as required for \cref{thm:main} based on our univariate approximation \cref{thm:cvx-approximation}.

Recall that the set of functions
\begin{align}
    \Scal_{ \lcvx , \delta } =   \hat{W}(\mu, m ) \cup \{ \relu_{it} | i \in [m/t]\},
\end{align} 
for $m = \frac{1}{\delta}$, $t = m^{1/3}$ and $\mu = 1/(12 m^{1/3} \log m)$, $\delta$-approximately spans $\Cvx$, the family of convex, Lipschitz functions on $[0,1]$.

\begin{corollary} \label{cor:approx_cvx}
    For $\delta > 0$ sufficiently small, the set of statistics $\Scal_{\lcvx, \delta }$ 
    gives
    \begin{align}
        (O(\log(1/\delta)^{4/3}/\delta^{2/3}), O(\log(1/\delta)^{4/3}/\delta^{2/3}) , \delta)
    \end{align}
    approximation to $\lcvx$. 
    
\end{corollary}

In fact, in the above theorem, carefully keeping track of the coefficients in \cref{thm:cvx-approximation}, we can bound $\lambda$ by $O(1)$ but we do not need this strengthening. 
Using the theorem above, we get the following result for learning an omnipredictor for the family of Lipschitz losses.
Given a class of functions $\Ccal$ denote by $\Bcal_{\mathrm{post}} $ the class of \tests\ obtained by postprocessing the functions in $\Ccal$ with an arbitrary bounded functions that is 
\begin{align}
    \Bcal_{\mathrm{post}, \delta} = \left\{ f \circ c : c \in \Ccal \quad f:\mathbb{R} \to \mathbb{R} \quad \abs{f(x)} \leq O(\log(1/\delta)^{4/3}/\delta^{2/3})    \right\}. 
\end{align}

The above theorem when combined with \cref{thm:main} gives the following result that states that calibration with respect to the family of statistics $\Scal_{\lcvx, \delta }$ and multiaccuracy with respect to the class $\Bcal_{\mathrm{post}, \delta}$ gives an omnipredictor for the family of convex, Lipschitz losses.

\begin{corollary} \label{thm:cvx_omni}
    Let $\Ccal$ be a hypothesis class. 
    Let $\epsilon \in (0,1) $ and $\rho , \sigma \in (0,1)$.  
    Given access to a $(\rho , \sigma )$-weak learner for $ \Bcal_{\mathrm{post} , \epsilon/4}$ with $\sigma \leq \rho \leq \eps^{4/3}/6$ and sample complexity $Z$, 
    there is an algorithm that runs in time and sample complexity
    \begin{align} 
        \tilde{O} \left(  2^{ \tilde{O}( \epsilon^{-2/3} )  }  +  \frac{Z}{\epsilon^{2/3} \sigma^2}  \right)
    \end{align} 
    
    that produces a $ \Scal_{\lcvx, \epsilon/4}$  predictor $p$ that is a $( \lcvx , \Ccal , \epsilon  )$ omnipredictor. \footnote{Here $\tilde{O}$ hides polylogarithmic factors in $1/\epsilon$.} 
    
\end{corollary}
\begin{proof}
    From \cref{cor:approx_cvx}, we have that the family of statistics $ \Scal_{\lcvx, \delta }$ gives an approximation to $\lcvx$ with parameters
    $(O(\log(1/\delta)^{4/3}/\delta^{2/3}), O(\log(1/\delta)^{4/3}/\delta^{2/3}), \delta)$.
    Setting $\delta = \epsilon / 4$  and plugging into \cref{thm:main_alg}, we get the desired result.
\end{proof}

%% file: sections/moments.tex
\subsection{Omniprediction via moments for low-degree loss functions} 
\label{sec:omni_moments}

Here we show how to obtain omniprediction for low-degree polynomial loss functions via sufficient statistics.
     We first define the family of low-degree loss functions. 
     Let $ \Mcal_d = \left\{  x^i : i \leq d   \right\}$ denote the family of monomials of degree at most $d$ in $x$.

    \newcommand{\lpoly}{\Lcal^{\mathrm{poly}}_{d, \lambda , \delta}}
    \newcommand{\lpolyk}{\Lcal^{\mathrm{poly}}_{d, \lambda , \delta, k}}

    \newcommand{\Lpoly}{\Lcal^{\mathrm{poly}}_{d, \lambda , \delta}} 
    \newcommand{\Lpolyk}{\Lcal^{\mathrm{poly}}_{d, \lambda , \delta, k}}

     \begin{definition} \label{def:low-deg-loss} 
        Let $ \Lcal^{\mathrm{poly}}_{d, \lambda , \delta} $ be the family of loss functions for which $\Mcal_d$ forms a set of sufficient statistics i.e. $\Mcal_d$ gives $(d, \lambda, \delta)$-uniform approximations to $\Lpoly$. 
        Explicitly, for each $  \ell \in \Lpoly  $, there exits $r_i^{\ell} : \Acal \to \mathbb{R} $ for $i \leq d$ such that 
        \begin{align}
            \abs{  \ell(y, t)  - \sum_{i= 0}^d r_i^{\ell}(t) y^i } \leq \delta
        \end{align} 
        and 
        \begin{align}
            \sum_{i=0}^d \abs{r_i^{\ell}(t)} \leq \lambda. 
        \end{align}   
       
        Furthermore, if $r_{i}^{\ell}$ is a polynomial of degree at most $k$, we say that $ \left( \Mcal_d , \Rcal \right)  $ gives  $(d, \lambda, \delta, k)$-uniform polynomial approximations to $\ell$. We denote this subclass of loss functions as $\Lpolyk$.
        
     \end{definition}

    The main example of losses in this class are the $\ell_p$ losses for even $p$.
    That is, 
    \begin{align}
        \ell_p(y, t) = \left( y-t \right)^p. 
    \end{align} 

    First, we will look at the basic representation of this family of loss functions.
    Note that 
    \begin{align}
        \left( y-t \right)^p = \sum_{i=0}^p \binom{p}{i} (-t)^i y^{p-i}. 
    \end{align}
    Thus, we have 
    \begin{align}
        r_i^{\ell_p}(t) = \binom{p}{i} (-t)^i
    \end{align}
    and 
    \begin{align}
        \sum_{i=0}^p \abs{r_i^{\ell_p}(t)} = \sum_{i=0}^p \binom{p}{i} \abs{t}^i \leq \sum_{i=0}^p \binom{p}{i} = 2^p.
    \end{align}
    
    Clearly, for $p \leq k$, we have that $\Rcal$ is a family of polynomials of degree at most $k$.
    Thus, instantiating \cref{thm:main_alg}, in this setting we have a $ \exp(k^2) \epsilon^{-k}$ time algorithm to get omnipredictors for a class that includes all $\ell_p$ losses for $p < k$.
    Our main result in this section is an improved bound using better uniform approximations.

We present an improved approximation with degree $ \sqrt{p} $ and coefficients of size $2^{\sqrt{p}}$.   
The following theorem is a standard application of Chebyshev approximations, but we include a proof for completeness.
First, recall that the Chebyshev polynomial of degree $j$ is defined as
\begin{align}
    T_j (x) = j \sum _{k=0}^{j}(-2)^{k}{\frac {(j+k-1)!}{(j-k)!(2k)!}}(1-x )^{k}
\end{align}
and is a key tool in approximation theory. 
The following is a important result from approximation theory. We include a proof of the following lemma in \cref{sec:chebyshev_lem} for completeness.

\begin{lemma}
\label{thm:chebyshev}
    For any $n \in \mathbb{N}$ and $ \epsilon \in (0,1) $, there exists a polynomial $q$ of degree $ d=  \sqrt{n \log(1/ \epsilon )} $ such that 
    \begin{align}
        \abs{ x^n - q(x) } \leq \epsilon
    \end{align}
    for all $x \in \left[ 0,1 \right]$. 
    Further, $q$ can be represented as 
    \begin{align}
        q(x) = 2^{1-n} \left[     \sum_{ j \leq d } \binom{n}{\frac{n-j}{2}} \cdot T_j (x) +  \mathbb{I} \left[ n \equiv 0 \pmod 2  \right] \cdot \frac{T_0 (x)}{2} \binom{n}{ \frac{n}{2} } \right]
    \end{align}
    where $T_j$ is the degree $j$ Chebyshev polynomial. 
\end{lemma}

\begin{lemma} \label{lem:lpapprox}
    For $ p \leq n $, we have that $  \ell_p \in  \lpolyk    $ for $ k = d = \sqrt{ n \log(1 / \delta) }$ and $  \lambda =  d^3 2^{4d}$.   
\end{lemma}
\begin{proof}
    Recall that from \cref{thm:chebyshev}, we have for $d = \sqrt{n \log(1 / \delta)}  $ that 
\begin{align}
    (x-t)^n  &\approx_{\delta}  {2^{1-n}  \sum_{ j \leq d } \binom{n}{\frac{n-j}{2}} \cdot T_j (x - t)  } \\ 
    &=  2^{1-n}  \sum_{ j \leq d } \binom{n}{\frac{n-j}{2}} j \sum _{k=0}^{j}(-2)^{k}{\frac {(j+k-1)!}{(j-k)!(2k)!}}(1-x +t )^{k} \\ 
    & =  2^{1-n}  \sum_{ j \leq d }  \sum _{k=0}^{j}  \binom{n}{\frac{n-j}{2}} j (-2)^{k}{\frac {(j+k-1)!}{(j-k)!(2k)!}}(1-x + t)^{k} \\ 
    &= 2^{1-n}  \sum_{ j \leq d }  \sum _{k=0}^{j}  \binom{n}{\frac{n-j}{2}} j (-2)^{k}{\frac {(j+k-1)!}{(j-k)!(2k)!}}  \sum_{ h=0  }^k \binom{k}{h} (-x)^k (1+t)^{k-h} \\ 
    & =  2^{1-n}  \sum_{ j \leq d }  \sum _{k=0}^{j} \sum_{ h=0  }^k \binom{n}{\frac{n-j}{2}} j (-2)^{k}{\frac {(j+k-1)!}{(j-k)!(2k)!}}   \binom{k}{h} (-x)^k (1+t)^{k-h} 
\end{align}

Note that each of the terms in the above is bounded by $ 2^{4d} $, 
To see this note that $ 2^{-n }  \binom{n}{\frac{n-j}{2}} \leq 1 $, $ \frac {(j+k-1)!}{(j-k)!(2k)!}  = \frac{1}{j+k } \binom{j+k}{2k} \leq 2^{ j+k} \leq 2^{2d}  $ and $  \binom{k}{h} \leq 2^k \leq 2^d$
Therefore, we get the sum of the coefficients to be bounded by $ d^3 2^{ 4d } $ and the degree is bounded by $ O\left(\sqrt{n \log(1/\delta) } \right) $.
\end{proof}

    Given a class of functions $\Ccal$  define the set of \tests\ obtained by composing the functions with monomials of degree at most $k$ as $\Ccal^k$ i.e. 
    \begin{align}
        \Ccal^k = \left\{  c^j: j \leq k \quad c \in \Ccal \right\}.
    \end{align}
    This class of tests allows us to state the omniprediction result from \cref{thm:main_alg} for the particular case of low degree losses.
\begin{theorem}[Omniprediction for low degree losses] \label{thm:low_degree}
    Let $\Ccal$ be a hypothesis class. 
    Let $\epsilon \in (0,1) $, $\rho , \sigma \in (0,1)$, $\lambda \in \mathbb{R}$  and $ d, k \in \mathbb{N}  $.   
    Given access to a $(\rho , \sigma )$-weak learner for $ \Ccal^k $ with $\sigma \leq \rho \leq \eps /6 \lambda$ and sample complexity $Z$, there is an algorithm that runs in time
    \begin{align}
        \tilde{O} \left( d \lt(\frac{\lambda}{\eps} \rt)^{d+5} +  \frac{d}{\sigma^2} Z\right)
    \end{align}  
    and outputs an $( \lpolyk , \Ccal , \epsilon  )$ omnipredictor. 
    \end{theorem}
    In light of \cref{lem:lpapprox}, this gives an improved omnipredictor for the class of $\ell_p$ losses for $p \leq n$ and $p$ even. 
    We formally state this in the following corollary.

    \begin{corollary}[Omniprediction for $\ell_p$ losses]
        Let $\Ccal$ be a hypothesis class. 
        Let $\epsilon \in (0,1) $ and $\rho , \sigma \in (0,1)$.
        For all $n \in  \mathbb{N} $ 
        let $d = \sqrt{n \log ( 1 / \epsilon)}   $. 
        Given access to a $(\rho , \sigma )$-weak learner for $ \Ccal^d$ with $\sigma \leq \rho \leq \eps d^{-3} 2^{-4d} /6$ and sample complexity $Z$, there is an algorithm that runs in time
        \begin{align}
            {O} \left(  2^{ O (  n \log^2(1 / \epsilon) )  } +  \frac{ \sqrt{n \log(1/ \epsilon)} }{\sigma}  Z  \right)
        \end{align}
        and produces an $ \left(  \left\{   \ell_{p}  \right\}_{p \leq n ; p \text{ even} } , \Ccal , \epsilon  \right) $-omnipredictor.
    \end{corollary}

The dependence on $ \lambda = 2^{n}$ in the above can be improved by switching to a different basis of polynomials instead of the moment basis to get $\lambda =  d $ as in \cref{cor:bound_coeff} but we state it the above theorem in terms of the monomial basis due to the natural interpretation in terms of moments.

%% file: sections/bregman.tex
\subsection{GLM Loss Minimization} 
\label{sec:omni_GLM}

In this section, we consider the problem of minimizing losses corresponding to generalized linear models (GLM).
This family of losses arises in many natural machine learning applications due to their intimate connections to exponential families, graphical models and Bregman divergences. 
In particular, regression using GLM losses corresponds to maximum likelihood estimation when the data generating model is assumed to be an exponential family.
Given these connections, GLM losses have been extensively studied in statistics and machine learning \cite{mccullagh2019generalized,jordan2003introduction,wainwright2008graphical}.

In the setup of GLM loss minimization, the action space is $\Acal = [-1,1]^d$. 
Note that this deviates from the results in the previous subsections where the actions were one dimensional.
Let $g$ be a convex loss function and let $\Scal = \{s_i \}_{i \leq d} $ be a family of functions, which we will refer to as the family of statistics. 
Define the loss function 

    \begin{align}
        \ell_{g , \Scal } (y, t ) = g(t) - \sum_{i=1}^d s_i(y) t_i. 
    \end{align}
    Here $t \in\Acal$ and $g: \mathbb{R}^d \to \mathbb{R} $ is a convex function.
    $\ell_{g, \Scal}$ is referred to as the generalized linear model loss corresponding to $g$ and $\Scal$.
    Define the class of loss functions 
    \begin{align}
        \mathcal{L}_{\Scal, \mathrm{GLM}} = \{   \ell_{g, \Scal } : g \text{ is a convex function and } g(t) \text{ is bounded in } [-1,1]   \}   
    \end{align}
    as the set of GLM losses with statistics $ \Scal$. 
    Note that the bound on $t$ and the boundedness of $g$ holds for many loss functions of interest. 
    In situations that the boundedness does not hold, we can approximate the loss function by a function such that it holds.

    See \cite[Section 5]{gopalan2022loss} for an extended discussion on GLM losses.
    
    First, we note that the family of statistics $\Scal$ forms a $( d, d+1 , 0 )$-uniform approximation for the loss function $ \Lcal_{\Scal, \mathrm{GLM}}$.

    \begin{theorem} \label{thm:glm-uniform}
        Let $ \Scal $ be a family of statistics and let $d = \abs{\Scal}$.
        Then, $\Scal$ forms a $( d, d + 1 , 0 )$-uniform approximation for the class of losses $ \Lcal_{\Scal, \mathrm{GLM}}$ with coefficient functions $r_i(t) = - t_i$ and $r_0(t) = g(t)$.
    \end{theorem}
    \begin{proof}
        From the definition of $\Lcal_{\Scal, \mathrm{GLM}}$, we have that 
        \begin{align}
            \ell_{g, \Scal} (y , t) =  g(t) - \sum_{i=1}^d s_i(y) t_i
        \end{align}
        Also, recall that we use the convention $ s_{0} \in \Scal $ is the constant function. 
        Thus, we have that $ \ell_{g, \Scal}  $ is approximated by the statistics $\Scal$ with error $0$ and coefficient functions $ r_{i} (t) = -  t_i  $ and $r_{0} (t) = g(t)$. 
        Further, we have for all $t$  
        \begin{align}
            \abs{g(t)} + \sum_{i=1}^d \abs{t_i} \leq 1 + d 
        \end{align} 
        as required. 
    \end{proof}

    The main theorem that we show in this section is that for the class of GLM losses, we can compute an omnipredictor using calibration and multiaccuracy. 
    The key aspect of this result that distinguishes it from the results in previous subsections is that the set of \tests\ for which multiaccuracy is required is the original class of hypothesis $\Ccal$.

        Note that in the setting of GLM loss minimization it is natural to consider a class of function $\Ccal$ consisting of functions $c $ whose output is in $d$ dimensions. 
    For such a class of functions it is natural to define multiaccuracy coordinatewise i.e. 
    \[ \abs{\E_{\mD^*}[(s_i(\y^*) - p_i(\x))c_i(\x)]} \leq \alpha. \]
    where $c_i$ is the $i$th coordinate of $c$. 
    The algorithmic result in \cref{thm:main_alg} can be extended to this setting by assuming a weak learner for each coordinate and in the below theorem we will refer to this simply as a weak learner for the class $\Ccal$.

\begin{theorem} \label{thm:GLM}
    Let $ \Scal $ be a family of statistics and let $d = \abs{\Scal}$.
    Let $\Ccal$ be a class of functions in $ \Xcal \to [0,1]^d$. 
    Let $\epsilon \in (0,1) $, $\rho , \sigma \in (0,1)$.  
    Let $p $ be a $\Scal$ predictor that is that is $(\Ccal , \epsilon/6d)$-multiaccurate and $\epsilon/12d$-calibrated with respect to $\Scal$.
    Then, $p$ is an $ (\Lcal_{ \Scal, \mathrm{GLM} }  , \Ccal , \epsilon) $ omnipredictor. 
    Further, given access to a $(\rho, \sigma)$ weak learner for class $\Ccal$ with $\sigma \leq \rho \leq \eps/12d$ and sample complexity $Z$, such an omnipredictor can be computed in time and sample complexity 
    \begin{align}
        O\left(  d \left(\frac{d}{\epsilon}\right)^{d+5} + \frac{d Z}{ \sigma^2 }    \right). 
    \end{align}
    
\end{theorem}

\begin{proof}

    From \cref{thm:glm-uniform}, we have that $\Scal$ forms a $( d, d + 1 , 0 )$-uniform approximation for the class of losses $ \Lcal_{\Scal, \mathrm{GLM}}$ with coefficient functions $r_i(t) = - t_i$ and $r_0(t) = g(t)$.
    Further, note that from \cref{rem:omni} and \cref{thm:main}, we have that a $\Scal$ predictor that is $ \Rcal \circ \Ccal = \Ccal $ multiaccurate and calibrated is an omnipredictor for the class of losses $ \Lcal_{\Scal, \mathrm{GLM}}$.
    The algorithmic claim then follows from \cref{thm:main_alg}. 
    \end{proof}

    Note that this theorem generalizes the result from \cite{gopalan2022loss} which corresponds to the one-dimensional case where the set of statistics was $ \Scal =  \{1,y\}$.
    \cite{GGKS23} relate the problem of omniprediction for the one-dimensional GLM case to the problem of agnostically learning single index models. 
    In independent and concurrent work, \cite{noarov2023highdimensional} obtain results for omniprediction in the multidimensional GLM case where the means are the sufficient statistic. 
    The approach by \cite{noarov2023highdimensional} does not focus on omniprediction for general loss classes but can be used to obtain online omniprediction results for the multidimensional GLM case. 
    For an extended discussion, see \cite[Section 6.3.2]{noarov2023highdimensional}.

%% file: sections/algorithms.tex
\section{Calibrated multiaccuracy for statistic predictors}
\label{sec:omni_alg}

In this section, we will address the algorithmic question of designing omnipredictors for loss functions approximated by families of statistics. 
As we saw in \cref{sec:omni_Oi}, in order to achieve omniprediction, we need to find a predictor that is both calibrated and multiaccurate. 
In \cref{sec:calibration_alg}, we will design an algorithm that produces a calibrated predictor for a family of statistics. 
In \cref{sec:cal_mult_alg}, we will design an algorithm that 
produces a predictor that in addition is multiaccurate with respect to a class of \tests\ $\Bcal$.

\subsection{Calibrating \texorpdfstring{$d$}{d}-dimensional statistics} \label{sec:calibration_alg}

As before, we will denote by $\Scal$ the family of statistics that we would like to produce calibrated predictors for.
Denote by $d$ the cardinality of the family of statistics $\Scal$.
A predictor for $\Scal$ is a function $p: \Xcal \to [-1,1]^d$ where the $i$th coordinate corresponds to the prediction for the $i$th statistic. 

Let $\delta > 0 $ denote a discretization parameter.
We will construct predictors that predict vectors of integer multiples of $\delta$. 
We partition the range of the \(d\)-dimensional predictor \([-1,1]^d\) into \(m = \lceil 1/\delta \rceil^d\) distinct subsets, denoted by \(\{ \mathcal{V}_1, \ldots, \mathcal{V}_m \}\).
For any \(d\)-dimensional vector \(j = (j_1, \ldots, j_d)\), where the element $j_i$ is an integer in the interval  \([-\lceil 1/\delta \rceil, \lceil 1/\delta \rceil - 1]\), each subset \(\mathcal{V}_j\) is the Cartesian product of intervals $[j_1\delta, (j_1+1)\delta] \times \cdots  [j_d \delta, (j_d + 1) \delta] $. 
We will refer to the set of all such $j$ by $\Jcal_\delta$.

 We can associate any $\Scal$-predictor $p$ with two predictors. 
 Denote by \(p^\delta\) the predictor which rounds the predictions of $p$ to integer multiples of \(\delta\), that is,  
\[
p^\delta (x) = j \delta, \text{ where } j \text{ is such that } p(x) \in \mathcal{V}_j
\]
and a calibrated predictor \(\bar{p}\)
\[
\bar{p}(x) =  \E[s(\y)|p(x) \in \mathcal{V}_j], 
\]
Note that the predictor $p^\delta$ predicts vectors that $\delta$ close (in $\ell_\infty)$ to the predictions of $p$ for every $x \in \Xcal$. 
While $\bar{p}$ is the result of recalibrating $p^\delta$, the entries of $\bar{p}(x)$ would not necessarily be multiples of $\delta$.

Define $\ece(p)$ to be the expected calibration error of $p$. That is, 
\[
\ece(p) = \E_{\mD^*}\lt[\norm{\E[s(\y^*)|p(\x)] - p(\x)}_\infty \rt]. 
\]

We define the following norms on the space of $\Scal$-predictors:
\begin{align}
\ell_1 (p, q) &= \E \left[\norm{p(\mathbf{x}) - q(\mathbf{x})}_1\right] \\
\ell_2 (p, q) &= \E \left[\norm{p(\mathbf{x}) - q(\mathbf{x})}_2^2\right]^{1/2} \\
\ell_\infty (p, q) &= \max_{x \in \Xcal} \left[\norm{p(\mathbf{x}) - q(\mathbf{x})}_\infty\right]. 
\end{align}
Then, $\ell_\infty (p, p^\delta) \leq \delta$ and the expected calibration error of $p^\delta$, \( \ece(p^\delta) = \E [ \| p^\delta (\x) - \bar{p} (\x) \|_\infty ] \). 
Therefore, we can estimate the $\ece(p^\delta)$ from an empirical estimate of $\bar{p} (\x)$.

\begin{lemma}\label{lem:est-cal}
    Let $\mu, \delta \in [0,1]$ and $d$ be the dimension of the family of statistics $\Scal$.  Given access to an $\Scal$-predictor $p$ and random samples from $\mD^*$,
    \begin{itemize} 
    \item There exists an algorithm $\estece(p, \mu)$ which returns an estimate  of $\ece(p^\delta)$ within additive error $\mu$. The algorithm runs in time and sample complexity $\tilde{O}(d/(\delta^d \mu^3))$. 
    \item There exists an algorithm $\recal(p, \delta)$ which returns a predictor $\hat{p}$ which has $\ece(\hat{p}) \leq \delta$ and $l_1(\bar{p}, \hat{p}) \leq \delta$. The algorithm has time and sample complexity $\tilde{O}(d/\delta^{d+3})$.
    \end{itemize}
\end{lemma}
 
The main idea for $\estece$ is to collect enough samples, $O(d \log^2 (d/\delta)/\delta^d \mu^3))$, so that we can estimate the calibration error within each prediction bin $j \in \Jcal_\delta$ with high accuracy. For bins that hold significant weight in the distribution $\Dcal$, i.e $\Pr_\Dcal [p^\delta(\x) = j \delta] \geq \mu\delta^d/4$, the collected samples are large enough such that the empirical statistics are within a $\mu/4$-margin of the true statistic with high probability. We ignore bins with smaller proportions since their total contribution to the overall calibration error is at most $\mu/4$.

Similarly, for $\recal$, we collect enough samples, $O(d \log^2 (d/\delta)/\delta^{d+3}))$, so that for prediction bins that hold significant weight in the distribution, the empirical statistics are within a $\delta/4$-margin of the true statistic with high probability. We construct the predictor $\hat{p}$ to simply output the empirical statistic for each prediction bin.

\begin{algorithm}
\caption{Estimate Expected Calibration Error (\(\estece\))}
\textbf{Input:} Predictor $p: \Xcal \rightarrow [-1,1]^d$, Error parameter $\mu$, Discretization parameter $\delta$, Dimension $d$ \\
\textbf{Output:} Estimate of $\ece(p^\delta)$
\begin{algorithmic}[1]
\STATE $n \gets O(d \log^2(d/\delta)/(\delta^d \mu^3))$ 

\STATE Collect $n$ samples $\{(x,y) \}$ from $\mD^*$
\STATE Initialize $\widehat{\ece} \gets 0$
\FOR{each $j \in \Jcal_\delta$}
    \STATE Aggregate samples $T_j = \{(x,y) \mid p(x) \in \Vcal_j\}$ 
    \STATE Compute $\bar{s}_j = \frac{1}{|T_j|} \sum_{(x,y) \in T_j} s(y)$
    \STATE Compute $\varepsilon_j = \|\bar{s}_j - j \delta\|_\infty$
    \STATE $\widehat{\ece} \gets \widehat{\ece} + \frac{|T_j|}{n} \varepsilon_j$
\ENDFOR
\RETURN $\widehat{\ece}$
\end{algorithmic}
\end{algorithm}

\begin{algorithm}
\caption{Recalibrate Predictor ($\recal$)}
\textbf{Input:} Predictor $p: \Xcal \rightarrow [-1,1]^d$, Discretization Parameter $\delta$ \\
\textbf{Output:} Recalibrated predictor $\hat{p}$
\begin{algorithmic}[1]
\STATE $n \gets O(d \log^2(d/\delta)/(\delta^d \mu^3))$ 

\STATE Collect $n$ samples $\{(x,y) \}$ from $\mD^*$
\FOR{each $j \in \Jcal_\delta$}
    \STATE Aggregate samples $T_j = \{(x,y) \mid p(x) \in \Vcal_j\}$ 
    \STATE Compute $\bar{s}_j = \frac{1}{|T_j|} \sum_{(x,y) \in T_j} s(y)$
    \FOR{each $x$ such that $p(x) \in \Vcal_j$}
        \STATE Set $\hat{p}(x) = \bar{s}_j$
    \ENDFOR
\ENDFOR
\RETURN $\hat{p}$
\end{algorithmic}
\end{algorithm}

\begin{proof}[Proof of Lemma \ref{lem:est-cal}]
We take $n = O(d\log^2(d/\delta)/(\delta^d \mu^3))$ random samples $(x, y)$ from $\mD^*$ to construct an empirical estimate $\hat{p}$ of $\bar{p}$. For each $j \in \Jcal_\delta$, let $T_j$ refer to the set of samples $(x,y)$ such that $p(x) \in \Vcal_j$ and $n_j$ the number of such samples. Define the value
\begin{align}
\label{def:yj}
    \bar{s}_j &= \frac{1}{n_j} \sum_{(x,y) \in T_j} s(y) \\
    \varepsilon_j &= \|\bar{s}_j - j \delta \|_\infty \notag \\
    \estece &= \sum_{j}^{[1/\delta]^d} \frac{n_j}{n}\varepsilon_j\notag
\end{align}
The algorithm returns the value $\estece$ as estimate for $\ece(p^\delta)$. 

Now we show that $\estece$ is $\mu$-close to $\ece(p^\delta)$ with high probability. We ignore any values of $j$  with small proportions, that is, $\Pr[p^\delta(\x) = j \delta] \leq \mu\delta^d/4$, since all such values only contribute $\mu/4$ to $|\estece - \ece(p^\delta)|$ with probability $0.1$. Call the other values of $j$ large. For every large $j$,  we have by Chernoff bounds, we have
\begin{align}
    \Pr[n_j \leq C(d \log(d/\delta)/\mu^2\delta^d)]  \leq \frac{\delta}{30}
\end{align}
Assuming this event holds, we have
\begin{align}
     \Pr\left[\left|\Pr[p(\mathbf{x}) \in \Vcal_j] - \frac{n_j}{n} \right| \geq \frac{\mu}{4}\right] \leq \frac{\delta}{30}\\
     \Pr\left[ \|\bar{s}_j - \mathbb{E}[s(\y) \mid p(\mathbf{x}) \in \Vcal_j] \|_\infty \geq \frac{\mu}{4} \right] \leq \frac{\delta}{30}.
\end{align}

We take a union bound over all $[1/2\delta]^d$ possible large values. 
Except with error probability $0.2$, none of the bad events considered above occur, and we have $|\estece - \ece(p^\delta)| \leq \mu$. We can reduce the failure probability by repeating the estimator and taking the median. For simplicity, we ignore the failure probability.

To define the predictor $\hat{p}$, we repeat the analysis above with $\mu = \delta$. We define $\hat{p}(x) = \bar{s}_j$ if $p(x) \in \Vcal_j$.  We show that it is close to $\bar{p}$ in $\ell_1$. The contribution of small values of $j$ to $\E[|\bar{p}(\x) - \hat{p}(\x))|]$ is no more than $\mu/4$. For large buckets, we have 
\[ |\bar{s}_j - \bar{p}(x)| \leq  \lt|\bar{s}_j - \E[s(\y)|p(\x) \in \Vcal_j ]\rt| \leq \delta/2 + \mu/4. \]
Thus overall, the distance is bounded by $(\delta/2 + \mu/4) \leq \delta$ by our choice of $\mu$. 

Lastly, we bound the calibration error, using the fact that $\bar{p}$ is perfectly calibrated, and $\hat{p}$ is close to it $\bar{p}$. Note that both $\bar{p}$ and $\hat{p}$ are constant on all $x \in p^{-1}(\Vcal_j)$. Hence
\begin{align} 
\ece(\hat{p}) &= \E_{\mb{\Vcal_j}}\lt| \E_{p(\x) \in \mb{\Vcal_j}}[s(\y) - \hat{p}(\x)] \rt| \\
& \leq   \E_{\mb{\Vcal_j}}\lt| \E_{p(\x) \in \mb{\Vcal_j}}[s(\y) - \bar{p}(\x)] \rt| + \E_{\mb{\Vcal_j}}\lt| \E_{p(\x) \in \mb{\Vcal_j}}[\bar{p}(\x) - \hat{p}(\x)] \rt| \\
&= \E[|\bar{p}(\x) - \hat{p}(\x)|] \\
& \leq \delta
\end{align}
as required. 

\end{proof}

\subsection{Calibrated Multiaccuracy for \texorpdfstring{$d$}{d}-dimensional statistics} \label{sec:cal_mult_alg}

In this section, we will design algorithms that produce multiaccurate predictors for a class of tests $\Bcal$ . 

The algorithm will assume access to a weak learning oracle.

\begin{definition}[Weak agnostic learning]
    Let $\Bcal$ be a class of \tests.   
    For parameters $\rho > \sigma > 0$, a $(\rho, \sigma)$-weak learner $\wl_{\Bcal}$ is an algorithm $\wl_{\Bcal}$ specified with the following input-output behavior.
    The input is a  function $f: \Xcal \rightarrow [-1,1]$ which $\wl_{\Bcal}$ is given access to through samples $(x, z)  \sim \Dcal_{f} $ where $\Dcal_{f}$  corresponds to a distribution such that $\x \sim \Dcal_\Xcal$ and $\E [\z | \x = x] = f(x)$.  $\wl_{\Bcal}$ outputs either $b \in \Bcal$ or $\bot$ such that the following conditions hold: 
\begin{itemize}
\item If the output is $b \in \Bcal$, then $\E [b(\x) f(\x)] \geq \sigma$.
\item If there exists any $b' \in \Bcal$ such that $\E [b'(\x) f(\x)] \geq \rho$, then the output cannot be $\bot$.

\end{itemize}

The number of samples drawn from $\Dcal_f$ during the execution of the algorithm is referred to as the sample complexity of $\wl_{\Bcal}$ and the running time of the algorithm is defined in the natural way. 
\end{definition}

In our algorithms for calibrated multiaccuracy, $f(x)$ will take the form of $\frac{1}{2} (\E [s_i(\y)|\x] - p_{t,i}(x))$ for $i \in [d]$ and $(\x,\y) \sim \Dcal $ where $\Dcal$ is a corresponding to the original problem.
Here $p_{t,i}(x)$ refer to the predictions of the $i$-th statistic $s_i$ of a predictor $p_t(x)$.
Note that $\frac{1}{2} (\E [s_i(\y)|\x] - p_{t,i}(x)) \in [-1,1]$. 
In order to simulate sample access to $\Dcal_f$, we draw a sample $(\x, \y) \sim \mD$. Then we generate $\z \in \pmo$ so that $\E[\z] = \frac{1}{2} (s_i(\y)  - p_{t,i}(\x))$. Since $\frac{1}{2} (s_i(\y) - p_{t,i}(\x)) \in [-1,1]$, this uniquely specifies the distribution of  $\z$. Moreover
\[
    \E[\z|\x] = \frac{1}{2} (\E[s_i(\y)|\x] - p_{t,i}(\x)) = f(\x).
\]
Though we don't explicit allow for this in the above definition, some weak learners take as input real-valued labels; in this case, we can use $\z = \frac{1}{2}(s_i(\y) - p_{t,i}(\x))$ to label $\x \sim \Dcal_\Xcal$. 
Since, we are alternate between the two models of access through sampling, we will not elaborate on this further.

\paragraph{Multiaccuracy.}

A main ingredient in our algorithm is the algorithm for multiaccuracy provided in \cite{hebert2018multicalibration}. 
Although it is designed to achieve multiaccuracy for a single mean predictor in the boolean setting, it works for any one-dimensional statistic predictor $q: \Xcal \rightarrow [-1,1]$. The algorithm in \cite{hebert2018multicalibration} assumes access to a $(\rho, \sigma)$-weak learner for $\Bcal$ and behaves iteratively as follows: 
Starting with an arbitrary predictor $q_0$, it iteratively updates the predictor using a \tests\ $b \in \Bcal$ that correlates with the predictor. This step is repeated until no such hypothesis exists.

\begin{algorithm}
\label{alg:wmc}
\caption{Multiaccuracy for one-dimensional statistic predictors ($\mathsf{MA}$)}
\textbf{Input:}  Predictor $q_0: \X \rightarrow [-1,1]$\\
Error parameter $\alpha \in [0,1]$. \\
Oracle access to a $(\rho, \sigma)$ Weak learner $\wl$ for $\Bcal$ under $\mD_\Xcal$ where $\alpha \geq \rho$ .\\
\textbf{Output:}  Predictor $q_T$.
\begin{algorithmic}
\STATE $t \gets 0$
\STATE $ma \gets \mathsf{false}$
\WHILE{$\neg ma$}
    \STATE $b_{t+1} \gets \wl(\frac12 (q^* - q_t))$. 
    \IF{$b_{t+1} = \bot$}
        \STATE $ma \gets \mathsf{true}$ 
    \ELSE
        \STATE $h_{t+1} \gets q_t + \sigma b_{t+1}$.
        \STATE $q_{t+1} \gets \Pi(h_{t+1})$ (where $\Pi$ projects $h_{t+1}$ onto $[-1,1]$).
        \STATE $t \gets t+1$.
    \ENDIF
\ENDWHILE
\RETURN $q_t$.
\end{algorithmic}
\end{algorithm}

\begin{lemma}\cite{hebert2018multicalibration}
    \label{lem:hkrr}
    Let  $\Bcal$ be a class of \tests\ and $\alpha \in \left[ 0,1 \right] $ be an accuracy parameter. 
    There exists an algorithm $\ma (q_0, \Bcal, \alpha)$ that takes a one-dimensional predictor $q_0 :\Xcal \rightarrow [-1,1]$ and returns a predictor $q_T$ for $T \geq 0$ where $q_T$ is $(\Bcal, \alpha)$-multiaccurate and
    \[ l_2(q^*, q_0)^2 -  l_2(q^*, q_T)^2 \geq T\sigma^2,\]
    where $q^*(x) = \E[s(\y) | \x]$ is the true predictor of the one-dimensional statistic.
\end{lemma}

\paragraph{Calibrated Multiaccuracy}

We now present our algorithm for finding a predictor that is both calibrated and $(\Bcal, \alpha)$-multiaccurate with respect to a family of statistics $\Scal$. It follows the same outline as the algorithm for calibrated multiaccuracy in \cite{gopalan2022loss}.
We will set the discretization parameter to $\delta$ to be small compared to $\alpha$ (concretely we will choose $\delta = \alpha^2/32$).
Informally, the algorithm may be viewed as starting with an arbitrary predictor $p_0: \Xcal \rightarrow [-1,1]^d$ and iteratively improving it by alternating the following steps. 

\begin{enumerate}
    \item Multiaccuracy: 
    
    \begin{enumerate}
        \item For each dimension $i \in d$, run $\ma (p_{t, i}, \Bcal, \alpha)$ to obtain $p_{t+1,i}$ 
        
        \item Set $p_{t+1} = [p_{t+1,1}, p_{t+1,2}, \ldots, p_{t+1,d}] $ and compute the discretization $p_{t+1}^\delta$
    \end{enumerate} 
    \item Calibration:
    \begin{enumerate}
        \item Estimate the calibration error of $p_{t+1}^\delta$ using $\estece$.
        \item If the calibration error is low, return the predictor $p^\delta_{t+1}$.
        \item If the calibration error is large, recalibrate it to $\hat{p}_{t+1}$, using $\recal$ and return to the multiaccuracy step.

    \end{enumerate}
\end{enumerate}

We formally present this as  Algorithm learnOmni below.

\begin{algorithm}
    \label{alg:map}
    \caption{$\mathsf{learnOmni}$}
    \textbf{Input:} 
    Parameters $d, \lambda, \eps$ for which $\Scal$ gives $(d, \lambda, \eps)$ uniform approximations to $\Lcal$ \\
    $\Scal$-predictor $p_0: \X \rgta [-1,1]^d$ \\
    Coefficient family $\Rcal: \{ r: \Acal \rightarrow \reals$ \} \\
    Hypothesis Class $\Ccal = \{ c: \Xcal \rightarrow \Acal \}$ \\
    
    Oracle access to a $(\rho, \sigma)$-Weak learner $\wl$ for $\Rcal \circ \Ccal$ under $\Dcal_\Xcal$ where $\sigma \leq \rho \leq \eps/12\lambda$ .\\
    
    \textbf{Output:} $\Scal$-predictor $q_T$.\\
    \begin{algorithmic}
    \STATE $\alpha \gets \eps/6d$
    \STATE $\beta \gets \eps/6\lambda$
    \STATE $\delta \gets \beta^2/32$
    
    \STATE $q_0 \gets p_0$
    \STATE $ma \gets \mathsf{false}$ 
    \STATE $t \gets 0$
    \WHILE{$\neg ma$}
        \STATE $t \gets t+1$
        \FOR{each dimension $i \in d$}
            
            \STATE $p_{t, i} \gets \ma(q_{t, i}, \Rcal \circ \Ccal, \alpha - \delta)$
        \ENDFOR
        \STATE $p_t \gets \left[ p_{t,1}, p_{t,2}, \ldots, p_{t,d} \right]$
        
        \IF{$\estece(p_{t}, \beta/4) > 3\beta/4$}
            \STATE $q_{t} \gets \recal(p_{t}, \delta)$
        \ELSE
            \STATE $q_{t} \gets p_{t}^\delta$
            \STATE $ma \gets \mathsf{true}$
        \ENDIF
        
    \ENDWHILE
    \STATE return $q_t$ 
    \end{algorithmic}
    \end{algorithm}

Note that if we terminate, we output a predictor that achieves both multiaccuracy and calibration, as required.
The main part of the analysis is showing that the algorithm terminates in a small number of iterations.
The key observation is that both steps reduce the potential function $\ell_2(p^*, p_t)^2$, which (for suitable choices of parameters) allows us to bound the overall number of iterations.
We capture the overall complexity of the algorithm in the following theorem.

\mainAlg*

Notably, the number of calls to the weak learner is polynomial in $d$. This is in contrast to achieving multicalibration, for which the best known algorithm requires an exponential (in $d$) number of calls to the weak learner.

Our proof of Theorem \ref{thm:main_alg} will rely on some results from \cite{gopalan2022loss}. 

\begin{corollary} [Generalization of Corollary 7.5 in \cite{gopalan2022loss}]
\label{cor:err-red}
For the predictors $p^\delta, \bar{p}$ defined above, 
\begin{align}
\label{eq:third-sq}
        l_2(p^*, p^\delta)^2 - l_2(p^*, \bar{p})^2 & \geq \ece(p^\delta)^2.
\end{align}
where $p^*(x) = \E[s(\y)|\x = x]$
\end{corollary}
Since this is a generalization of the result from \cite{gopalan2022loss}, we present its proof in Appendix \ref{app:proofs} for completeness.

\begin{lemma}[\cite{gopalan2022loss}]
\label{lem:ell1}
For any predictors $p_1, p_2$ such that $l_1(p_1,p_2) \leq \delta$,
\begin{align}
    \label{eq:first-sq}
        \lt| l_2(p^*, p_1)^2 - l_2(p^*, p_2)^2\rt| &\leq 2\delta.
    \end{align}
Further, if $p_1$ is $(\mC, \alpha)$-multiaccurate, then $p_2$ is $(\mC, \alpha + \delta)$-multiaccurate.
\end{lemma}

\begin{proof}[Proof of \cref{thm:main_alg}]
First we show that the predictor $q_t$ returned by the learnOmni algorithm is $(R \circ C, \alpha)$-multiaccurate. By construction, $q_t$ only predicts multiples of $\delta$ and is $\delta$-close (in $\ell_\infty$ norm) to $p_t$ which is $(R \circ C, \alpha - \delta)$-multiaccurate. Thus, by \Cref{lem:ell1}, $q_t$ is $(\Rcal \circ \Ccal, \alpha)$-multiaccurate.

Observe that the predictor $q_t$ is $\beta$-calibrated. learnOmni terminates if $\estece (q_t, \beta/4) \leq 3\beta/4$. Thus, the calibration error of $q_t$ is at most $\beta$.

Since $q_t$ is both $\beta$-calibrated and $(R \circ C, \alpha)$-multiaccurate, by \Cref{thm:main}, it is an $(\Lcal, \Ccal, 3d\alpha + 3\lambda \beta + 3\eps)$-omnipredictor. By our choice of $\alpha = \eps/6d, \beta = \eps/6\lambda$, $q_t$ is an $(\Lcal, \Ccal, 4\eps)$-omnipredictor.

Now we show that the number of calls to the $(\rho, \sigma)$-weak learner is bounded by $O(d/\sigma^2)$. When we set  $p_{t, i} = \ma(q_{t, i}, \Ccal_i, \alpha - \delta)$, this results in $N_{t,i}$ calls to the weak learner. Thus, by \Cref{lem:hkrr},
\begin{align} 
    l_2(p^*, q_{t-1,i})^2 - l_2(p^*, p_{t,i})^2 \geq N_{t,i} \sigma^2.
\end{align}
Summing over $i \in [d]$, we have 
\begin{align} 
\label{eq:prog1}
    l_2(p^*, q_{t-1})^2 - l_2(p^*, p_{t})^2 \geq \sum_{i \in [d]} N_{t,i} \sigma^2.
\end{align}
since $\ell_2(p^*,p_t)^2 = \sum_{i \in [d]} \ell_2(p^*,p_{t,i})^2$. In total, we make $\sum_{i \in [d]} N_{t,i}$ calls to the weak learner to obtain $p_t$ from $q_{t-1}$.
We wish to bound the number of loops $T$. To do so, we use the fact that every time the calibration error of $p_t^\delta$ is large and we have to recalibrate, our potential function $\ell_2 (p^*, p_t)$ increase by a good amount. Concretely, if $\estece (p_t, \beta/4) \geq 3\beta/4$, then by \Cref{lem:est-cal}, $\ece (p_t^\delta) \geq 3\beta/4 - \beta/4 = \beta/2$.
Since $q_t = p_t^\delta$, applying \Cref{cor:err-red} gives
\begin{align}  
\label{eq:prog2}
    l_2(p^*, p_t)^2 - l_2(p^*, q_t)^2  \geq \ece (p^\delta_t)^2 - 4\delta \geq \frac{\beta^2}{8}.
\end{align}
Adding Equations \eqref{eq:prog1} and \eqref{eq:prog2}, for $t \in \{1, \ldots, T-1\}$,
\[ l_2(p^*,q_{t-1})^2 - l_2(p^*,q_t)^2 \geq \frac{\beta^2}{8}. \]
Summing this over all $t$, 
\[ l_2(p^*,q_0)^2 - l_2(p^*,q_{T-1})^2 \geq (T-1)\frac{\beta^2}{8}. \]
Since $q_0 = p_0$ and $l_2(p^*,q_{T-1})^2 \geq 0$, we have
\[ T \leq 1 + \frac{8}{\beta^2}{l_2(p^*,p_0)^2} = O(1/\beta^2). \]
 To bound the number of calls to the weak learner, we sum Equation \eqref{eq:prog1} over all $t \in [T]$, and Equation \eqref{eq:prog2} over all $t \leq T -1$ to get
\[ l_2(p^*,p_T)^2 - l_2(p^*, p_0)^2 \geq \sum_{t,i} N_{t,i} \sigma^2 + (T-1)\frac{\beta^2}{8}. \]
This implies that 
\[ \sum_{t,i} N_{t,i} \leq d/\sigma^2 \]
Since the number of calls to the weak learner in loop $t$ is bounded by $\sum_{i \in [d]} N_{t,i} + d$, we bound the number of calls by
\[ \sum_{t,i} (N_{t,i} + 1) \leq \frac{d}{\sigma^2} + T = O(d/\sigma^2) \] 
since $T= O(1/\beta^2)$  and $\beta \geq \rho \geq \sigma$.

\end{proof}

%% file: sections/related.tex
\section{Further Related work}
\label{sec:rel_works}

\paragraph{Multi-group fairness.} The fairness notions of multiaccuracy and multicalibration were introduced in the influential work of Hebert-Johnson, Kim, Reingold and Rothblum \cite{hebert2018multicalibration}, see also the work of \cite{KleinbergMR17, KearnsNRW18}. There has been a large body of followup work, extending it to the regression setting \cite{JLPRV21}, to other notions of calibration \cite{GopalanKSZ22} and much more. Connections between multicalibration and boosting are established in the works of \cite{gopalan2021omnipredictors, Globus-HarrisHRS23}. The recent work of \cite{BlasiokGHKN23} shows that multicalibration for neural networks can be obtained from squared loss minimization. The elegant  work of \cite{OI} introduced the notion of outcome indistinguishability and related it to multigroup fairness notions of varying strength in the Boolean setting. This work was extended beyond the setting of Bernoulli labels by \cite{bernoulli}. Further connections between pseudorandomness and multigroup fairness were discovered in the work of \cite{DworkLLT23}, who also prove some new omniprediction results. Outcome indistinguishability was used to construct multigroup agnostic learners in the work of \cite{RothblumY21}. 

\paragraph{Omniprediction.} The work of \cite{gopalan2021omnipredictors} introduces the notion of omniprediction. The subsequent work of \cite{gopalan2022loss} brings the outcome indistinguishability perspective to omniprediction, introducing a general technique based on a simulated distribution that we generalize to the regression setting in this work.  Reverse connections between generalizations of omniprediction and multigroup fairness notions were established in the work of \cite{GKR23}. Omniprediction in a constrained setting where the predictor is required to satisfy other constraints which might be motivated for instance by fairness was considered in the works of \cite{HuNRY23, GGJKMR23}. Omnipredictors for performative prediction were studied in \cite{KimP23}. The problem of omniprediction in the online rather than  the batch setting was recently studied in the work of \cite{GJRR24}. The recent work of \cite{GGKS23} uses calibrated multiaccuracy to give the first agnostic learning guarantees for Single Index Models (SIMs), with respect to the squared loss.

\paragraph{Approximate rank.}  The notion of approximate rank of matrices arises naturally in communication complexity.
Both the $\epsilon$-error randomized and quantum communication complexities of a function are lower bounded by constant times the log of the approximate rank of the communication matrices (see \cite[Chapters 4 and 5]{LSbook}). 
It further turns out that this notion is closely connected to notions such as factorization norms \cite{linial2007lower, LeeS09}, hereditary discrepancy \cite{matouvsek2020factorization} and sign rank \cite{AlonLSV13}. 
In addition, the notion turns up in fundamental algorithmic problems such as density subgraph and approximate Nash equilibria.
See \cite{AlonLSV13} and references therein for further discussion.
Most of the mentioned works above focus on the case of sign matrices.
Another area that is generally related to the notions considered in our work is the area of approximation theory (see \cite{carothers1998short, szeg1939orthogonal}) but most work in the area focused on approximations in terms of functions families such as polynomials and rational functions. 
To our knowledge, our work is the first to study the notion of approximate rank (and thus approximation in terms of an arbitrary basis of functions) for bounded convex functions on $\intzo$.

%% file: sections/appendix.tex
\section{Proof of \cref{cor:bound_coeff} }
\label{app:proofs}

 \cref{cor:bound_coeff} follows immediately by applying the following proposition. We provide a proof, which uses John's theorem \cite[Theorem 3.1]{ball1997elementary} for completeness.

\begin{fact}
\label{fact:john}
     Let $ \Fcal$ be a family of functions with domain $\mathcal{X}$ such that for all $f \in \Fcal$, we have $\abs{f(x)} \leq 1 $.
     Let $\{ g'_i \}_{i\leq d}   $ be a family of functions that $\epsilon$-approximately spans $\Fcal$.
     Then, there exists $\{ g_i \}_{i\leq d}   $ with $\abs{g_i(x)} \leq 1 $ that $\epsilon$-approximately span $\Fcal$ with coefficients $\alpha_i  $ satisfying 
     \begin{align}
         \sum_i \abs{\alpha_i (f) } \leq (1+ \epsilon)  d 
     \end{align}
     for all $f \in \Fcal$. 
\end{fact}

\begin{proof}[Proof of Fact \ref{fact:john}]
    Let $f \in \Fcal$. 
    Let $\alpha'(f)$ be the vector such that 
    \begin{align}
        \abs{f(x) -\sum_i \alpha'_i(f) g'_i (x)} \leq \epsilon. 
    \end{align}
    Note that $\alpha'(f) \in \mathbb{R}^d $ for each $f$. 
    Let $K$ be the convex hull of $ \{ \pm \alpha'(f) \}_{f \in \Fcal}$. 
    This is a symmetric convex body in $\mathbb{R}^d$. 
    Thus, by John's theorem \cite[Theorem 3.1]{ball1997elementary}, there exists a linear transformation $T$ such that $ B_2 \subset  T(K) \subset \sqrt{d} B_2   $ where $B_2 = \{ x \in \mathbb{R}^k : \norm{x}_2 \leq 1\}  $ is the unit ball.

    Consider the coefficients $\alpha(f) = T \alpha'(f) \in TK$ and functions $  g_i (x)= \sum_j (T^{- 1})_{j,i} g'_j(x)  $. 
    The fact that $\{g_i\}_{i\leq d}$ $\epsilon$-approximately span $\Fcal$ with coefficients $\alpha$ follows by construction.
    Note that 
    \begin{align}
        \sum_i |\alpha_i | \leq \sqrt{d} \norm{ \alpha }_2 \leq d . 
    \end{align}
    
    Further, note that 
    \begin{align}
        \abs{g_i(x)} &\leq \sqrt{ \sum_i  (g_i(x))^2 } \\
        & \leq \sup_{  \gamma \in T(K)  }   \abs{ \left\langle \gamma_i ,  (g_1(x) , \dots , g_d(x)) \right\rangle  }\\
        &\leq \sup_{\gamma \in K } \abs{ \left\langle T \gamma ,  T^{-1}    (g_1'(x) , \dots , g'_d(x)) \right\rangle  }   \\ 
        &\leq \sup_{\gamma \in K } \abs{\sum_i \gamma_i g'_i(x)    } \\ 
        & \leq \sup_{ \alpha'(f) }  \abs{\sum_i \alpha'_i(f) g'_i(x)    } \\ 
        & \leq \sup_{f \in \Fcal} \abs{ f(x) + \epsilon  } \leq 1+ \epsilon
    \end{align}
    The inequality bounding the $\ell_2$ norm by the supremum over $T(K)$ follows because $ B_2 \subset T(K) $. 
    We get the desired bound by scaling down $g$ by $(1+\epsilon)$ and scaling up $\alpha_i$ by $(1 + \epsilon)$. 
    
\end{proof}

\begin{proof}[Proof of Corollary \ref{cor:err-red}]
\eat{
We first show that
\begin{align}
    \label{eq:second-sq}
    \E[(p^*(\x) - p^\delta(\x))^2] - \E[(p^*(\x) - \bar{p}(\x))^2] & = \E[(\bar{p}(\x) - p^\delta(\x))^2.
\end{align}}

We express the left-hand side of Equation \eqref{eq:third-sq} as the difference of expectations of the 2-norms involving $p^*, p^\delta, \bar{p}$
\begin{align}
    \mathbb{E} \left[ \left\| p^*(\x) - p^\delta(\x) \right\|_2^2 \right] - \mathbb{E} \left[ \left\| p^*(\x) - \bar{p}(\x) \right\|_2^2 \right] 
    &= \mathbb{E}\left[ \langle \bar{p}(\x) - p^\delta(\x), 2p^*(\x) - p^\delta(\x) - \bar{p}(\x) \rangle \right]
\end{align}
To see this, observe that
\begin{align}
&\mathbb{E} \left[ \left\| p^*(\x) - p^\delta(\x) \right\|_2^2 \right] - \mathbb{E} \left[ \left\| p^*(\x) - \bar{p}(\x) \right\|_2^2 \right] \\
&= \mathbb{E} \left[ \langle p^*(\x) - p^\delta(\x), p^*(\x) - p^\delta(\x) \rangle \right] - \mathbb{E} \left[ \langle p^*(\x) - \bar{p}(\x), p^*(\x) - \bar{p}(\x) \rangle \right] \\
&= \mathbb{E} \left[ \langle p^*(\x), p^*(\x) \rangle - 2\langle p^*(\x), p^\delta(\x) \rangle + \langle p^\delta(\x), p^\delta(\x) \rangle \right] \\
&\quad - \mathbb{E} \left[ \langle p^*(\x), p^*(\x) \rangle - 2\langle p^*(\x), \bar{p}(\x) \rangle + \langle \bar{p}(\x), \bar{p}(\x) \rangle \right] \\
&= \mathbb{E}\left[ \langle \bar{p}(\x) - p^\delta(\x), 2p^*(\x) - p^\delta(\x) - \bar{p}(\x) \rangle \right]
\end{align}

We consider the distribution on subspaces induced by choosing $\x \sim \mD$ and $\mb{\Vcal_j} \ni p(\x)$. Since $p^\delta$ and $\bar{p}$ are constant for each subspace $\Vcal_j$, we can write $p^\delta(\Vcal_j)$ and $\bar{p}(\Vcal_j)$ for their values in this subspace without ambiguity. Hence by first taking expectations over $\mb{\Vcal_j}$ and then $p(\x) \in \mb{\Vcal_j}$
\begin{align}
    \mathbb{E}\left[ \langle \bar{p}(\x) - p^\delta(\x), 2p^*(\x) - p^\delta(\x) - \bar{p}(x) \rangle \right]
    &= \mathbb{E}_{\mb{\Vcal_j}}\left[
    \langle \bar{p}(\mb{\Vcal_j}) - p^\delta(\mb{\Vcal_j}), \mathbb{E}_{\x|p(\x) \in \Vcal_j}[2p^*(\x) - p^\delta(\mb{\Vcal_j}) - \bar{p}(\mb{\Vcal_j})] \rangle \right] \\
    &= \mathbb{E}_{\mb{\Vcal_j}}\left[ \left\| \bar{p}(\mb{\Vcal_j}) - p^\delta(\mb{\Vcal_j}) \right\|_2^2 \right]\\
    &= \mathbb{E}_{\x \sim \mathcal{D}}\left[ \lt\| \bar{p}(\x) - p^\delta(\x) \right\|_2^2 \rt].
\end{align}

where the final line uses $\E[p^*(\x)|\x \in \Vcal_j] = \bar{p}(\Vcal_j)$. 
    
Since $p^\delta, \bar{p}$ are both constant one each subspace $\Vcal_j$, we have
\begin{align}
    \mathbb{E}_{\x \sim \mathcal{D}}\lt[ \lt\| \bar{p}(\x) - p^\delta(\x) \rt\|_2^2 \rt]
    &= \mathbb{E}_{\mb{\Vcal_j}}\lt[ \mathbb{E} \lt[ \lt\| \bar{p}(\x) - p^\delta(\x) \rt\|_2^2 \mid p(\x) \in \Vcal_j \rt] \rt]\\
    &= \mathbb{E}_{\mb{\Vcal_j}}\lt[ \mathbb{E} \lt[ \lt\| s(\y) - p^\delta(\x) \rt\|_2^2 \mid p(\x) \in \Vcal_j \rt] \rt]\\
    &\geq \mathbb{E}_{\mb{\Vcal_j}}\lt[ \mathbb{E} \lt[ \lt\| s(\y) - p^\delta(\x) \rt\|_2 \mid p(\x) \in \Vcal_j \rt] \rt]^2\\
    &\geq \ece (p^\delta)^2
\end{align}

where the first inequality uses the convexity of $x^2$. \eat{We can rewrite Equation \eqref{eq:first-sq} as
\[ \E[(p^*(\x) - p(\x))^2] - \E[(p^*(\x) - p^\delta(\x))^2] \geq -2\delta.\]
Adding these inequalities gives the desired claim.}
\end{proof}

\section{Proof of \cref{thm:chebyshev}} \label{sec:chebyshev_lem}

\begin{proof}
    Let $T_j$ be the degree $j$ Chebyshev polynomial. 
Recall that 
\begin{align}
    T_j (x) = j \sum _{k=0}^{j}(-2)^{k}{\frac {(j+k-1)!}{(j-k)!(2k)!}}(1-x )^{k}
\end{align}
and that for $\abs{x} \leq 1$, we have
\begin{align}
    \abs{T_j(x)} \leq 1. 
\end{align}

Further, we have the representation (see \cite{szeg1939orthogonal,chebyshev})
\begin{align}
    x^n = 2^{1-n} \left[    \sum_{ j \equiv n  \pmod 2 ;  j \neq 0 } \binom{n}{\frac{n-j}{2}} \cdot T_j (x) +  \mathbb{I} \left[ n \equiv 0 \pmod 2  \right] \cdot \frac{T_0 (x)}{2} \binom{n}{ \frac{n}{2} }  \right]
\end{align}

We truncate this up to degree $ d = O(\sqrt{n \log(1/\epsilon)})$. 
The residual is 
\begin{align}
    \abs{2^{1-n}  \sum_{ j \geq d } \binom{n}{\frac{n-j}{2}} \cdot T_j (x) } \leq 2^{1- n } \sum_{ j \geq d } \binom{n}{\frac{n-j}{2}} \leq \epsilon 
\end{align}
The last step follows from the Chernoff bound.
\end{proof}